\documentclass{article}

\usepackage{PRIMEarxiv}
\usepackage{authblk}

\usepackage[utf8]{inputenc} 
\usepackage[T1]{fontenc}    
\usepackage{hyperref}       
\usepackage{url}            
\usepackage{booktabs}       
\usepackage{amsfonts}       
\usepackage{nicefrac}       
\usepackage{microtype}      
\usepackage{fancyhdr}       
\usepackage{graphicx}       
\graphicspath{{media/}}     

\usepackage{amsmath, amssymb, natbib, graphicx, url, algorithm2e, amsthm, bm, cleveref}
\usepackage{bbm}
\usepackage{mathtools}
\usepackage{color}
\usepackage{thm-restate}
\usepackage{times}
\usepackage{fluff} 

\usepackage{xcolor}

\allowdisplaybreaks

\newcommand{\Reg}{\operatorname{Reg}}
\newcommand{\actions}{\mathcal{A}}
\newcommand{\experts}{\mathcal{F}}
\newcommand{\A}{\mathcal A}
\newcommand{\F}{\mathcal F}

\newcommand{\opt}{\textsc{Opt}}

\newcommand{\EMD}{\mathrm{EMD}}
\newcommand{\kol}{\mathrm{Kol}}
\newcommand{\indicator}[1]{\mathbf{1}\!\left\{#1\right\}}

\newcommand{\innerprod}[2]{\left\langle #1,#2\right\rangle}

\newcommand{\Null}{\perp}

\newcommand{\eat}[1]{}

\title{Optimally Pacing Budget Spending and Learning}

\author[1,2]{Mark Braverman$^{*}$}
\author[1]{Jingyi Liu$^{*}$}
\author[2]{Jieming Mao}
\author[2]{Jon Schneider}
\author[1]{Eric Xue$^{*}$}

\affil[1]{Department of Computer Science, Princeton University}
\affil[2]{Google Research}

\begin{document}
\maketitle

\renewcommand{\thefootnote}{\fnsymbol{footnote}}

\renewcommand*{\thefootnote}{\arabic{footnote}}

\begin{abstract}
We establish near-optimal regret bounds for budget-constrained online learning against arbitrary classes of budget-pacing experts in the adversarial setting. In particular, given any class of $F$ experts and a candidate budget pacing schedule, we provide a full-information algorithm which obtains regret $O(D \sqrt{\log F}+ \sqrt{T\log F})$ against all experts whose cumulative spending stays within distance $D$ of this schedule, matching lower bounds established by Braverman et al. (2025).

We additionally show that our technique extends to various problems in online resource allocation, where the learner gets to see the rewards and costs of the current options available to them, and establish $O(D\sqrt{\log F})$ regret bounds when fractional allocation is allowed. This is the first algorithm we are aware of which can achieve $o(\sqrt{T})$ guarantees for such tasks.
\end{abstract}


\section{Introduction}

Consider first the classical full-information online learning problem with experts. There are $N$ actions, $F$ experts and $T$ rounds in total. In each round, the learner gets an advice distribution over the set of $N$ actions from each expert, and then they choose an action to take. Each action is associated with some bounded reward, unknown to the learner prior to decision. In the adversarial model, the reward can be chosen by an adaptive adversary who might know the learner's algorithm and the transcript of play but does not know the learner's private randomness in the current round. The learner's performance is measured in terms of their regret against the best expert, i.e. the difference between the cumulative rewards earned by the best expert and the cumulative reward earned by the learner. This setting is well understood and the celebrated Multiplicative Weights Update (MWU) algorithm is known to achieve an optimal regret guarantee of $O(\sqrt{T\log F})$\citep{cesa-bianch2006prediction}.

In many real-world applications, such as repeated auctions with budgets, drone control, and resource allocation, playing an action not only yields a reward but also incurs some cost. For example, in a repeated auction, the action may be a bid, the reward may correspond to winning an item, and the cost is the resulting payment.
This motivates the study of online learning problems with resource constraints (OLRC) \citep{stradi2025noregret}, where each action additionally consumes a bounded resource cost, and the learner is endowed with a fixed budget of this resource and cannot overspend. Once the budget is depleted, the learner must stop and only take the null action going forward, which gives zero reward and incurs zero cost. The learner's goal is thus to minimize regret against the best expert while maintaining budget feasibility. 

In contrast to the online learning problem without resource constraints, where taking any action in the current round does not affect the learner's ability to take any action in a future round, the online learning problem with resource constraints forces a learner to only choose the null action once their budget is depleted, which  introduces a temporal coupling of the learner's decision sets across the time horizon. Moreover, different experts may have very different spending patterns over time. The learner therefore cannot freely change its mixture over experts across rounds: even if every expert is individually budget-feasible, switching among them can cause the learner itself to violate the budget constraint. These differences with the classical setting highlight the difficulty in designing an optimal learning algorithm for the setting with resource constraints: one must simultaneously 1) learn which action is better and 2) decide what mixture of experts to follow that will guarantee both no-regret and budget feasibility. 

In fact, this problem is already non-trivial even without the first difficulty. In the online resource allocation (ORA) setting \citep{stradi2025noregret, balseiro2023dual}, the current round's reward and resource costs of each action (and thereby of following each expert) are known to the learner before they need to make a decision, so the learner doesn't need to ``guess" which action/expert gets a better reward or a better reward to cost ratio for the current round. Thus ORA removes uncertainty about the current round, but not the intertemporal allocation problem: the learner still does not know whether budget spent today should instead be saved for better opportunities in the future.

The second difficulty is therefore best summarized by the ``spend-or-save" dilemma, present in both the ORA setting \citep{balseiro19learning} and the OLRC setting \citep{Immorlica22}. In the OLRC setting, it shows that the classical benchmark of competing against a strategy that plays a fixed distribution until the budget is depleted is no longer possible, and any algorithm that satisfies the budget constraint must incur a worst-case regret of $\Omega(T)$. \cite{braverman2025newbenchmark} characterized the learnability of the OLRC problem with experts in terms of the variability of their hindsight spending patterns: sublinear regret is attainable if and only if the experts are close to a common subpacing spending pattern in unnormalized Earth Mover's Distance (EMD), with EMD radius $o(T^2)$. Their lower bound construction is, in spirit, very close to the original ``spend-or-save" dilemma, where even with two experts, if their spending patterns $c, c' \in [0,1]^T$ differ by $\EMD(c, c')=\Omega(T^2)$ (for example, one expert spends all budget $B=T/2$ in the first half, and the other expert spends all budget $B=T/2$ in the second half), then any algorithm must incur a worst-case regret of $\Omega(T)$ on some instantiation of reward and cost sequence and expert advice distributions. However, for general $\EMD(c, c')$, their lower bound construction only gives a bound of $\Omega(\EMD(c, c')/T)$, which is below their algorithmic upper bound of $\tilde{O}(\sqrt{\EMD(c, c')}+\sqrt{T\log F})$ whenever $\EMD(c, c')=o(T^2)$. 

Thus the ``spend-or-save" dilemma and prior work on learnable benchmarks highlight that the key to quantifying learnability is through controlling the variability of the experts' spending patterns. Yet, to the best of our knowledge, existing results that propose new benchmarks do not have matching upper and lower bounds \citep{braverman2025newbenchmark, stradi2025noregret}. This motivates us to ask the question: \textbf{can we tightly characterize the regret rate in learning settings with resource constraints?}

\subsection{A Toy Example: Alternating Online Knapsack}

The above question is already non-trivial in very simple settings. To see this, consider the following toy problem. You have an incoming stream of $T$ balls whose colors alternate between red and blue (so every odd numbered ball is red, and every even numbered ball is blue). The learner processes the balls one by one: in round $t$, they observe the reward $r_t \in [0, 1]$ of ball $t$ and have to decide what fraction of it to pick up. The learner can pick up at most $T/2$ balls and wants low regret relative to either picking all the red balls or picking all the blue balls. How small can this regret be?

Note that this maps cleanly onto the online resource allocation (ORA) problem defined above: each ball has unit cost, and each color of ball (red and blue) corresponds to a different expert. The red expert has spending pattern $(1, 0, 1, 0, \dots)$, the blue expert has spending pattern $(0, 1, 0, 1, \dots)$, and they are both within EMD distance $O(T)$ of the uniform pacing pattern $(1/2, 1/2, \dots)$. The existing results of \cite{braverman2025newbenchmark} yield an $O(\sqrt{\EMD}) = O(\sqrt{T})$ regret algorithm against these experts, and indeed a simple variant of MWU (treating each pair of balls as a single round with a red arm and a blue arm) achieves this regret bound, even when the reward is not observable until the end of the round. On the other hand, the best lower bound construction only gives a lower bound of $\Omega(\EMD / T) = \Omega(1)$. 

We show that in this case the lower bound is tight: there exists an algorithm that achieves constant regret for this problem. The key idea is somewhat different from the usual Lagrangian treatment of a budget constraint.  A standard approach would attach a dual price to deviation from the fixed pacing budget and optimize reward minus priced resource consumption.  But regret is measured relative to one of the two experts, not relative to this fixed pacing baseline. Our algorithm instead measures \emph{both} reward regret and spending discrepancy relative to the same mixture of the red and blue experts.  It then chooses the current action by solving a small minimax problem over this comparator mixture and a resource-price variable. This lets a single potential control both quantities and, in this example, yields constant regret.

More generally, we can get similar results for the problem where there are alternating blocks of $D$ red balls and $D$ blue balls, for a total of $T$ rounds. Here the EMD-based regret bounds give a regret of $O\!\left(D\cdot \sqrt{T/(2D)}\right)=O(\sqrt{DT})$. But the above technique establishes a regret bound of $O(D+\sqrt{T})$ in the OLRC setting, and a regret bound of only $O(D)$ in the ORA setting. Note that for all $D =o(T)$, this is asymptotically better than the best previously known algorithm for both the OLRC and ORA settings.

\subsection{Our Results}

More generally, we answer the question of how to characterize instance-optimal regret rates in a general class of ORA and OLRC problems. In particular, we show that instead of earth mover's distance, the more relevant quantity is \emph{the largest difference in experts' cumulative spending over any prefix of the time horizon}. Formally, consider the following distance between two spending patterns $c,c' \in [0,1]^T$:
\[
    \operatorname{Kol}(c,c')
    :=
    \max_{t\in[T]}
    \left|
        \sum_{s=1}^t (c_s-c'_s)
    \right|.
\]


We call this quantity the \emph{Kolmogorov distance}. This captures directly the feature exploited by the lower-bound construction in the spend or save dilemma: if two spending patterns are far apart in earth mover's distance, then their cumulative spending must differ substantially at some point in time, and it is this cumulative-spending gap that drives the resulting spend-or-save lower bound construction. In particular, we will show that across multiple different settings the Kolmogorov distance between spending patterns of experts characterizes the optimal regret achievable. For example, in the toy example above, the Kolmogorov distance between the red spending pattern and the blue spending pattern is $1$, explaining the constant regret algorithm.

We therefore fix a baseline spending pattern $b$ and a radius $D$, and compete against experts whose spending patterns remain within Kolmogorov distance $D$ of $b$\footnote{Importantly, we do not need to know in advance which experts are $D$-pacing. We may start with an arbitrary candidate set of experts and prune experts online as soon as their cumulative spending deviation exceeds $D$: from that point onward, we replace their future cost by the baseline cost and their reward by $0$. We can do this by adding free-disposal actions with ``virtual costs" to the action set. Since our guarantees do not depend on the number of actions $N$, this is without loss of generality. The resulting experts are at most $D+1$-paced, so the same guarantees should hold asymptotically.}. We call such experts $D$-pacing.  In our main theorem, we show that we can establish a regret bound of either $O(D\sqrt{\log F})$ (in the ORA setting) or $O((\sqrt{T} + D)\sqrt{\log F})$ (in the OLRC setting) against the performance of the best of these pacing experts, when allocation is allowed to be fractional:

\begin{theorem}[Theorem~\ref{thm:core} (simplified)]\label{thm:intro_main}
Fix a time horizon $T$, a baseline spending pattern $b$, and a radius $D$. Assume there exists a constant $\alpha>0$ such that, for all times $t$ and actions $a$, we have a bounded reward-to-cost ratio\footnote{The theorem holds more generally under the soft constraint  $r_t(a)\le \alpha c_t(a)+\varepsilon_t$ with $\varepsilon_t\ge 0$ and $\sum_{t=1}^T \varepsilon_t \le V$ for some constant $V$.}
\[
r_t(a)\le \alpha\, c_t(a).
\]
Suppose there are $F$ experts, each of which is $D$-paced with respect to the same baseline spending pattern $b$.

Then there exist (fractional play)\footnote{In fractional play, choosing $p_t\in\Delta(\actions)$ yields reward $r_t(p_t)$ and cost $c_t(p_t)$ deterministically. In sampled play, the learner instead draws $a_t\sim p_t$ and receives the realized reward $r_t(a_t)$ and cost $c_t(a_t)$. For sampled play under full-information feedback, \Cref{sec:sampled-feasibility} gives online algorithms that satisfy the budget constraint pathwise and  with probability at least $1-\delta$, incur regret at most the corresponding fractional-play regret bound plus an additional
$\textstyle O\!\left(
      \sqrt{T\log(1/\delta)}
      +\log(1/\delta)
      \right)$ term.} online algorithms with the following guarantees:
\begin{itemize}
  \item \textbf{ORA:} $\Reg_T = O\!\left(D\sqrt{\log F}\right)$.
  \item \textbf{OLRC:} $\Reg_T = O\!\left((\sqrt{T}+D)\sqrt{\log F}\right)$.
\end{itemize}
\end{theorem}

These guarantees are optimal up to constant factors. In Section~\ref{sec:lower-bound}, we show an $\Omega(D\sqrt{\log F})$ lower bound even in the ORA setting.
Together with the classical $\Omega(\sqrt{T\log F})$ experts lower bound,
this also yields an $\Omega(\max\{D,\sqrt{T}\}\sqrt{\log F})$ lower bound in the OLRC setting, matching our upper bound up to constant factors.

The algorithms of Theorem~\ref{thm:intro_main} are flexible and can be extended in many ways: to algorithms with sampled play, to sets of experts which are \emph{subpacing}\footnote{An expert is $D$-subpacing if its spending pattern is dominated by a spending pattern which is $D$-pacing. We provide two extensions: one provides guarantees when the experts satisfy only a relaxed notion of $D$-pacing, which we call $D$-subpacing, while the other ensures that the resulting algorithm is itself $O(D)$-subpacing.}, and to settings with multiple resources. We provide a full summary of guarantees in the single-resource setting in \Cref{appendix:summary-of-guarantees}; the corresponding multiple-resource guarantees are given in Theorem~\ref{thm:core} and \Cref{sec:extensions}.

\subsection{Related Work}

\textbf{Online Resource Allocation (ORA)\;\& Online Learning with Resource Constraints (OLRC)}

In the ORA setting, the typical benchmark is the best-in-hindsight dynamic strategy that satisfies the global budget constraint\citep{balseiro2023dual}, and in the OLRC setting, the typical benchmark is the set of strategies that play the same distribution every round until they run out of budget \citep{Immorlica22}. In the OLRC setting, when the learner only gets bandit feedback, this is often referred to as the Bandits with Knapsacks problem, initiated by \cite{Badanidiyuru13} for the stochastic environment and \cite{Immorlica22} for the adversarial environment. When the rewards and costs are adversarially chosen (the main focus of our paper) there is a $\Omega(T)$ lower bound in both the ORA and OLRC settings \citep{balseiro19learning, Immorlica22}. Thus prior work mostly focuses on achieving competitive ratio guarantees for these settings \citep{Immorlica22, castiglioni2022online, castiglioni2022unifying, castiglioni2024online, balseiro2023dual}. Exceptions to these are works that focus on alternative benchmarks \citep{stradi2025noregret, braverman2025newbenchmark} and works that derive a regret guarantee as a function of some nonstationarity measure of the environment \citep{slivkins2023contextual, liu2022non, fikioris2023ApproximatelySB}. 
\cite{jenatton2016onlineoptimization} study 
dynamic regret guarantees for more general non-additive objective functions in the 1-lookahead setting (which is similar to our ORA setting), where the objective can encode a penalty term for long-term constraint violations. Our benchmarks are incomparable: their benchmark is the optimal dynamic sequence for a soft penalized objective, whereas we compete against a set of $D$-pacing experts under a hard budget constraint. 

A recurring theme in both the ORA setting and the OLRC setting is the use of the online primal-dual framework or Lagrangian-based algorithm for developing algorithms \citep{balseiro2023dual,Immorlica22, stradi2025noregret, braverman2025newbenchmark}. The budget constraint is Lagrangified by a dual multiplier, and the primal player makes decisions based on this Lagrangified objective, while the dual player updates the dual multiplier separately, which serves as a penalty for deviations from a uniform per-round budget \citep{balseiro2023dual} or a spending plan \citep{stradi2025noregret}. Our algorithm, interestingly, does not follow the same primal-dual template. Instead of using an alternating primal step and dual step, our algorithm optimizes the primal variable and dual variable jointly according to some potential function, which is what allows us to get a $T$-independent bound in the ORA setting and a decoupled $O(D+\sqrt{T})$ bound in the OLRC setting against $D$-pacing experts.

\textbf{Online Convex Optimization with Constraints}
There is a vast literature on online convex optimization problems with long-term constraints \citep{mannor2009online,mahdavi2012trading}. At round $t$, the learner makes a decision $x_t\in \mathcal{X}$  and gets the full-information feedback of a convex cost function $f_t$ and a constraint function $g_t\leq 0$. Different from the OLRC setting, which requires hard budget feasibility, this setting allows the algorithm to have constraint violations, and the goal is to minimize regret against some strategy $x^*$ (with certain feasibility assumptions) and the cumulative constraint violation separately. \cite{mannor2009online} provides an impossibility result for obtaining a  vanishing bound in both objectives when $x^*$ is only required to satisfy $\sum_t g_t(x^*)\leq 0$. Follow-up work often restricts the benchmark strategy to satisfy $g_t(x^*)\leq 0, \forall t\in [T]$ \citep{neely2017online, sun2017safety, chen2018bandit, sinha2024optimal}, which is analogous to a strict requirement for the benchmark strategies to follow a given spending plan in our setting. A slightly more general benchmark is considered by \cite{liakopoulos2019cautious}, where the benchmark strategy only needs to be feasible for $K$-sized windows. The objectives of this line of work differ from ours, as it does not hard-stop the algorithm when there is a constraint violation, so the results are incomparable.

\textbf{Bidding in Repeated Auctions}

Repeated auctions with budget constraints form a canonical single-resource instance of OLRC. Extensive literature employs pacing multipliers to dynamically shade bids and regulate expenditure \citep{borgs07dynamicsofbid,balseiro19learning, gaitonde2026budgetpacing}. While the standard benchmark is the best-in-hindsight multiplier, vanishing regret against this baseline is unattainable in the adversarial setting due to the spend-or-save dilemma \citep{gaitonde2026budgetpacing}. Consequently, recent works  \cite{gaitonde2026budgetpacing}; \cite{lucier24autobidders} consider more relaxed benchmarks---the best sequence of perfect pacing multipliers, where each perfect pacing multiplier makes sure the  budget (or ROI) constraint is tight in expectation in that round and achieve no-regret guarantees as a function of the path length of the perfect pacing multipliers. Our benchmark extends beyond strategies that need to satisfy per-round budget and scales smoothly with the variability of spending patterns in the benchmark class. 

\subsection{Comparisons with Prior Work}
\cite{braverman2025newbenchmark} study the adversarial Bandits with Knapsacks setting and they introduce a new benchmark class consisting of experts with spending patterns within some radius of a baseline spending pattern $b$, in Earth Mover's Distance. More precisely, for any two spending patterns $c, c'\in [0,1]^T$, the unnormalized Earth Mover's Distance between them is $\textstyle \mathrm{EMD}(c, c') \coloneq \sum_{t=1}^T \abs{\sum_{s=1}^t (c_s - c'_s)}$. The regret bound they achieve against a benchmark class with radius $\EMD$ is $\tilde{O}(\sqrt{\EMD}+\sqrt{T\log F})$. In comparison, we introduce the Kolmogorov distance $\kol(c,c'):=\max_{t\in[T]} \left|\sum_{s=1}^t (c_s-c'_s)\right|$. In the full-information setting, we achieve a regret bound of $O(D\sqrt{\log F} + \sqrt{T\log F})$ when the Kolmogorov radius is $D$. Since a Kolmogorov radius $D$ implies an Earth Mover's Distance of $\Omega(D^2)$ for spending patterns of the same total budget spending, we have $D = O(\sqrt{\mathrm{EMD}})$. Hence, our guarantees recover the dependence on EMD from prior work in the full-information setting, up to a $\sqrt{\log F}$ factor, while potentially giving strictly smaller regret when the maximum cumulative-spending deviation is much smaller than $\sqrt{\mathrm{EMD}}$.

\cite{stradi2025noregret} study both the ORA and OLRC settings with adversarial rewards and costs and a target spending plan, but our results differ in key ways. In the ORA setting (with sampled play), they achieve a dynamic regret of $\tilde{O}(\frac{1}{\rho_{\min}}\sqrt{T})$ against the best dynamic strategy that spends at most the per-round budget of the spending plan, where $\rho_{\min}$ is the minimum per-round budget of the spending plan. In the OLRC setting (with sampled play), they only benchmark against fixed distributions (that has per-round spending bounded above by the per-round spending of the budget plan), instead of dynamic strategies like ours, and achieve a regret of the same order $\tilde{O}(\frac{1}{\rho_{\min}}\sqrt{T})$. Our benchmark differs from theirs as we study dynamic strategies in both settings of a set of $D$-pacing experts, and instead of strict adherence to the spending plan, we only require the experts to approximate the plan in Kolmogorov distance (i.e. the maximum absolute cumulative difference in spending). A more robust version considered in \cite{stradi2025noregret} does allow strategies to satisfy the spending plan up to some error terms in each round, but the additional regret term depends linearly on the cumulative error. This is a looser parameter than the Kolmogorov distance we propose --- for example, a strategy that spends its budget $B=T/2$ uniformly over even rounds (1 per even round) is very close to a pacing baseline that spends uniformly $1/2$ each round in Kolmogorov distance, but is very far when measured by cumulative error, since it overspends by a constant in $\Omega(T)$ rounds. Our paper thus proposes a more refined metric, and we also show that this metric tightly characterizes the regret rate in both the ORA and OLRC settings.

\section{Models and Preliminaries}

\subsection{Problem Setup}

Consider a setting with $m$ resources, a set of $N$ actions, denoted $\actions$, and a set of $F$ experts, denoted $\experts$. The learner has a budget $B_i$ in each resource $i\in [m]$. Fix a time horizon $T\geq 1$. At each round $t\in [T]$, the learner receives an advice from each expert and picks an action $a_t\in \actions$. Each action $a\in \actions$ obtains some bounded reward $r_t(a)\in [0,1]$ and incurs a set of bounded costs $c_t^i(a)\in [0,1]$ for each resource $i\in [m]$. The rewards and costs are determined by an adaptive adversary. Let $r_t := (r_t(a))_{a\in \actions}$ and $c_t^i := (c_t^i(a))_{a\in \actions}$. The learner's goal is to maximize the total sum of rewards $\sum_{t\in [T]} r_t(a_t)$ from the selected actions subject to the budget constraint of each resource $\sum_{t\in [T]} c_t^i(a_t)\leq B_i, \forall i\in [m]$.\footnote{In fractional play, the learner's goal is to maximize the total sum of rewards $\sum_{t\in [T]} r_t(p_t)$ subject to the budget constraints $\sum_{t\in [T]} c_t^i(p_t)\leq B_i, \forall i\in [m]$.} 

We study two standard information structures.  In the \emph{online resource allocation (ORA)} regime, the round-$t$ reward and cost vectors $(r_t,c_t^1,\ldots,c_t^m)$ are revealed before the learner chooses its round-$t$ distribution.  In the \emph{online learning with resource constraints (OLRC)} regime, the learner chooses its round-$t$ distribution before observing these vectors. The OLRC regime additionally splits into the full-information setting and the bandit setting. In the full-information setting, the complete vectors $(r_t,c_t^1,\ldots,c_t^m)$ are revealed after the decision; in the bandit setting, only the reward and costs associated with the action $a_t$ the learner chooses is revealed, i.e. $(r_t(a_t),c_t^1(a_t),\ldots,c_t^m(a_t))$. In the paper, we also sometimes refer to the two regimes as with pre-decision feedback, corresponding to ORA, and without pre-decision feedback, corresponding to OLRC. We formally state the sequence of events in both regimes below. 

Our main analysis first considers the fractional relaxation, in which choosing
$p_t\in\Delta(\actions)$ yields deterministic reward $r_t(p_t):=\innerprod{p_t}{r_t}$ and deterministic
resource costs $c_t^i(p_t):=\innerprod{p_t}{c_t^i}$. We call this the \emph{fractional play}. Here the learner only needs to satisfy $\sum_{t\in [T]} c_t^i(p_t)\leq B_i, \forall i\in [m]$, which we call \emph{fractional budget feasibility}. In Section~\ref{sec:sampled-feasibility}, we return to the \emph{sampled play} setting, where the learner draws $a_t\sim p_t$ and needs to satisfy $\sum_{t\in [T]} c_t^i(a_t)\leq B_i, \forall i\in [m]$. Our results focus on the full-information feedback setting, since it already captures the key challenge of temporal coupling arising from the budget constraints.

\paragraph{Online resource allocation (ORA)}

At round $t=1,\dots, T$, 
\begin{enumerate}
    \item The learner receives an advice distribution $q_t^f\in \Delta(\actions)$ from each expert $f\in \experts$.
    \item For each action $a\in \actions$, the adaptive adversary decides on a reward $r_t(a)\in [0,1]$ and resource costs $c_t^{i}(a)\in [0,1], \forall i\in [m]$.
    \item The learner observes the full vectors $(r_t,c_t^1,\ldots,c_t^m)$.
    \item The learner chooses a distribution $p_t\in \Delta(\actions)$.
    \item In \emph{fractional play},  the learner earns reward $r_t(p_t)$ and costs $c_t^i(p_t), \forall i\in [m]$ deterministically. In \emph{sampled play}, the learner samples an action $a_t\sim p_t$ to take and earns reward $r_t(a_t)$ and costs $c_t^i(a_t), \forall i\in [m]$. 
\end{enumerate}

\paragraph{Online learning with resource constraints (OLRC)}

At round $t=1,\dots, T$, 
\begin{enumerate}
    \item The learner receives an advice distribution $q_t^f\in \Delta(\actions)$ from each expert $f\in \experts$.
    \item The learner chooses a distribution $p_t\in \Delta(\actions)$, observable by the adversary.
    \item For each action $a\in \actions$, the adaptive adversary decides on a reward $r_t(a)\in [0,1]$ and resource costs $c_t^{i}(a)\in [0,1], \forall i\in [m]$, not yet observable by the learner.
    \item In \emph{fractional play},  the learner earns reward $r_t(p_t)$ and costs $c_t^i(p_t), \forall i\in [m]$ deterministically. In \emph{sampled play}, the learner samples an action $a_t\sim p_t$ to take and earns reward $r_t(a_t)$ and costs $c_t^i(a_t), \forall i\in [m]$. In the full-information setting, the learner observes the full vectors $(r_t,c_t^1,\ldots,c_t^m)$. In the bandit setting, the learner only observes the reward and costs of the selected action, i.e. $(r_t(a_t), c_t^{1}(a_t), \dots, c_t^{m}(a_t))$. 
\end{enumerate}  

We assume there is a null action $\bot\in \actions$ with $r_t(\bot)=0$ and
$c_t^i(\bot)=0$ for every $i\in[m]$ and $t\in[T]$. The learner is required
to satisfy the hard budget constraints. Our algorithms enforce feasibility
conservatively using a hard stop: before round $t$, if the cumulative spending
on some resource $i$ plus the maximum possible one-round cost would exceed
$B_i$, the learner plays $\bot$ from round $t$ onward.


For $p\in\Delta(\A)$, write $r_t(p):=\innerprod{p}{r_t}$ and $c_t^i(p):=\innerprod{p}{c_t^i}$, and define $r_t^f:=r_t(q_t^f)$ and $c_t^{f,i}:=c_t^i(q_t^f)$.

We make the following assumption about the conversion of costs to rewards:
\begin{assumption}
\label{assumption:soft-conversion}
  For every resource $i$, there exist some constant reward-to-cost ratio $\alpha_i>0$, residuals $\varepsilon_t^i\ge0$, and some constant $V_i\ge0$ such that
\begin{equation*}
\textstyle r_t(a)\le\alpha_i c_t^i(a)+\varepsilon_t^i\quad\forall a\in \actions,t\in [T], \text{ and }\sum_{t=1}^T\varepsilon_t^i\le V_i.
\end{equation*}  
\end{assumption}

\subsection{Subpacing Experts and New Regret Benchmark}

Consider two spending patterns of the same resource $c=(c_t)_{t\in [T]}\in [0,1]^T$ and ${c'}=(c'_t)_{t\in [T]}\in [0,1]^T$. We say the Kolmogorov distance\footnote{The name is derived from the Kolmogorov-Smirnov Statistic \citep{gibbs20002onchoosing}, which measures the distance between two CDFs $F_1$ and $F_2$ by $\sup_{x\in \RR} |F_1(x)-F_2(x)|$.} between them is 
\[
\mathrm{Kol}({c}, {c'}) \coloneq \max_{t\in [T]} \abs{\sum_{s=1}^{t}(c_s-c'_s)}.
\]
Suppose $\mathrm{Kol}({c}, {c'}) \leq D$, then we say ${c}$ is \textit{$D$-pacing with respect to ${c'}$} (and ${c'}$ is also $D$-pacing with respect to ${c}$). We also allow for a more permissive definition, which takes into account the fact that ${c}$ and ${c'}$ might have very different total spending: we say ${c}$ is \textit{$D$-subpacing with respect to ${c'}$} if 
\[
\exists (d_t)_{t\in [T]} \text{ s.t. } d_t \in [0, c'_t] \; \forall t\in [T] \;\text{and}\; \kol({c}, {d}=(d_t)_{t\in [T]})\leq D.
\]
Fix a vector of baseline spending patterns $\bm{b}=(b^i)_{i\in[m]}$, where
$b^i=(b_t^i)_{t\in[T]}\in[0,1]^T$ and $\sum_{t=1}^T b_t^i=B_i$ for every
resource $i\in[m]$. We say an expert $f$ is $\bm{D}=(D_i)_{i\in[m]}$-subpacing
with respect to $\bm{b}$  if $c^{f,i}=(c_{t}^{f,i})_{t\in [T]}$ is $D_i$-subpacing with respect to $b^i$, for all resource $i\in [m]$.

Now consider a set of experts $\mathcal{F}$. Let $\mathcal F(\bm{D},\bm{b}):=\{f\in\mathcal F\mid f \text{ is $\bm{D}$-subpacing w.r.t. } \bm{b}\}$
be the set of $D$-subpacing experts in $\mathcal F$ with respect to a
prespecified baseline $b$. Our regret benchmark is
\[
\opt_{\bm{D}, \bm{b}, \mathcal{F}} \coloneq \max_{f \in \mathcal{F}(\bm{D}, \bm{b})} \sum_{t=1}^T r_t^{f}
\]
and the regret of a strategy $\mathbf{x} \coloneq (x_t)_t \in \Delta(\mathcal{A})^T$ against $\opt_{\bm{D}, \bm{b}, \mathcal{F}}$ is
\[
     \Reg_{\opt_{\bm{D}, \bm{b}, \mathcal{F}}}(\mathbf{x}) \coloneq \opt_{\bm{D}, \bm{b}, \mathcal{F}}- \sum_{t=1}^T r_t(x_t).
\]

\section{A Warm-up Example}
\label{sec:warm-up-example}

Before we present the full algorithm, we give some intuition in a simpler pre-decision-feedback
(ORA) setting with two experts, two actions, one resource, fractional play, and a perfectly
pacing baseline. In each round, the learner observes the current reward and cost before choosing
its action.\footnote{This is closer to the online resource allocation setting rather than the online learning setting, but the results for the two settings are closely connected.} 

\subsection{The Alternating Red and Blue Balls Problem}
Fix $D\in \mathbb{Z}^+$. Consider the setting where a sequence of $D$ red balls is followed by a sequence of $D$ blue balls, with this pattern repeated alternately for a total of T balls and $T=2kD$ for some $k\in \mathbb{Z}^+$. In each round, a value $r_t$ of the current ball is revealed, and the cost for each ball is always 1. The two actions are ``picking the ball" and the null action $\perp$. The red expert always advises picking the red balls, and the blue expert always advises picking the blue balls. The only thing that is hidden from the learner at round $t$ is the reward of future balls arriving at rounds $>t$. The budget of the learner is $T/2$ (so she can pick $T/2$ units of balls in total) and for simplicity, we assume the learner can pick fractionally. 

The baseline spending pattern ${b}$ for this problem is the perfectly pacing spending pattern, i.e. ${b}=\{1/2,\dots, 1/2\}\in [0,1]^T$. By definition, the red expert $R$ has a spending pattern of ${c}^R = (1,\dots, 1, 0,\dots,0 \dots )$, and the blue expert $B$ has a spending pattern of ${c}^B = (0,\dots, 0, 1,\dots,1 \dots )$. Since the alternating monochromatic blocks have size $D$ each, we have $\kol({c}^R, {b})\leq D/2$ and $\kol({c}^B, {b})\leq D/2$. We note that the red expert's reward at $t$ is $r_t^R = r_t \indicator{t \text{ is a Red ball}}$ and the blue expert's reward is $r_t^B = r_t \indicator{t \text{ is a Blue ball}}$. The learner is competing against the benchmark $\opt_{D, {b}, \mathcal{F}}$ where $\mathcal{F}=\{R,B\}$. Thus \[
\textstyle \opt_{D, {b}, \mathcal{F}} = \max\left\{\sum_{t\in [T]} r_t^R, \sum_{t\in [T]}r_t^B\right\}.\]
Let $x_t\in [0,1]$ denote the fraction picked by the learner in round $t$, i.e $[x_t, 1-x_t]$ is the distribution the learner chooses over the two actions. Then the learner's regret for using the strategy $\mathbf{x}=(x_t)_{t\in [T]}$ is $\Reg_{\opt_{D, {b}, \mathcal{F}}}(\mathbf{x}) \coloneq \opt_{{D}, {b}, \mathcal{F}}- \sum_{t=1}^T r_tx_t.$

\subsection{An Optimal Algorithm for the Alternating Red and Blue Balls Problem}\label{section:opt-red-blue}

In the following section, we design an algorithm with a regret guarantee of $O(D)$ and is $O(D)$-pacing against $b$. The regret guarantee is in fact optimal, by the lower bound in \Cref{theorem:lower-bound-any-spending-pattern}, since $\kol({c}^R, {c}^B)= \Theta(D)$.  

\medskip
\textbf{Step 1. Reduction to $D$ subproblems.}
In fact, we can reduce the problem to an even simpler problem, where the length of each monochromatic block is 1, and the baseline spending pattern is still the perfectly pacing spending pattern. We do so by putting the arriving red and blue balls into $D$ separate queues. For each ball that arrives, if it is the $j$-th ball in a monochromatic block, then we put it in the $j$-th queue. Thus each queue becomes a subproblem with $D'=1$, with both experts being constant-pacing with respect to the perfectly pacing baseline. It suffices to show that for each subproblem, the algorithm has a regret guarantee of $O(1)$ and a pacing guarantee of $O(1)$ against the perfectly pacing baseline.

\textbf{Step 2. Solving the Subproblem with $D'=1$.} We now focus on a single queue and, for notational simplicity,
relabel its horizon as $T$. Thus the balls alternate between red and blue, with a red ball followed by a blue ball,
$T$ is even, and the learner has budget $T/2$. At each round $t\in[T]$, the learner observes the current reward
$r_t\in[0,1]$ and then chooses a fraction $x_t\in[0,1]$ of the current ball to eat. The two benchmark experts eat
all red balls and all blue balls, respectively. The baseline spending pattern is uniform $1/2$ each round. This feedback model lies naturally between two extremes: 
\begin{itemize}
    \item If, upon the arrival of each red ball, the learner knew the rewards of both that red ball and the following blue ball before choosing what fraction to eat, then it could simply eat whichever of the two has a larger reward. It would consume exactly one unit per red--blue pair and obtain reward $\sum_{k=1}^{T/2}\max\{r_{2k-1},r_{2k}\}$, which is at least the reward of either the red expert or the blue expert. Hence it would have nonpositive regret, satisfy the budget exactly, and have only constant cumulative spending deviation from the uniform spending rate at any point.
    \item In contrast, if the learner had to choose $x_t$ before observing the current reward $r_t$, the classical online-learning lower bound yields worst-case regret of $\Omega(\sqrt{T})$.\footnote{For example, let the reward of every ball be drawn independently from
$\operatorname{Ber}(1/2)$.  If $x_t$ is chosen before observing $r_t$, then
$r_t$ is independent of $x_t$ conditional on the history, and hence
$\mathbb{E}[x_t r_t]=\frac12\mathbb{E}[x_t]$.  Under the budget constraint
$\sum_{t=1}^T x_t \leq T/2$, the learner therefore obtains expected reward $\leq T/4$.
On the other hand, the cumulative rewards $R$ and $B$ of the Red and
Blue experts are independent $\operatorname{Bin}(T/2,1/2)$ random variables.
Thus
$
\mathbb{E}[\max\{R,B\}]
=\frac{T}{4}+\frac12\mathbb{E}|R-B|
=\frac{T}{4}+\Theta(\sqrt{T})$,
so the expected regret is $\Omega(\sqrt{T})$.}
\end{itemize}  Our setting lies between these two cases: the learner knows the reward of the current ball before deciding how much to eat, but has no information about future rewards. We show that the ability to know the current round's reward is already sufficient to achieve constant regret while maintaining constant pacing.

Our goal is two-fold: minimize regret against both experts and minimize the Kolmogorov distance from the baseline. Since we need to prove worst-case guarantees for both objectives, we can assign a dual variable to each objective and view the problem as a zero-sum game between the learner and an adversary (separate from the adversarial environment), where the adversary picks the ``worst" dual variables that punishes the learner when regret against any expert is high or when the algorithm is deviating from the baseline by too much, and the learner needs to make decisions that avoid a high punishment. The payoff of this hypothetical zero-sum game needs to be carefully designed and merges the two objectives into a single function of the learner and the adversary's decisions. 

Let us begin by defining some notations.  
Let
$a_t^R=\indicator{\text{$t$ is Red}}$ and $a_t^B=\indicator{\text{$t$ is Blue}}$ be the fraction that the Red expert picks and the fraction that the Blue expert picks at round $t$, respectively.
Let
$R_t^R=\sum_{s\le t} r_s(a_s^R-x_s)$ and $R_t^B=\sum_{s\le t} r_s(a_s^B-x_s)$
denote the learner's regret against the Red expert and the Blue experts, at the end of round $t$. We write the regret vector as $R_t=(R_t^R,R_t^B)$.
We parameterize a distribution over the two experts by $p(z)=\left(\frac12+z,\frac12-z\right)$, where $
z\in\left[-\frac12,\frac12\right]$.
On round $t$, the corresponding distribution eats
$a_t(z)
=\innerprod{p(z)}{(a_t^R,a_t^B)}$, and $\innerprod{R_t}{p(z)}$ is the learner's regret to the
distribution $p(z)$. Let $Y_t=\sum_{s\le t}\left(x_s-\frac12\right)$ be the learner's cumulative overspending from the pacing baseline, and let $
E_t(z)=\sum_{s\le t}\bigl(x_s-a_s(z)\bigr)$ be the learner's overspending against the distribution $p(z)$.

Suppose on round $t$ the learner chooses $x\in[0,1]$, while an adversary
chooses the dual variables $z\in[-1/2,1/2]$ and a dual resource
price $v\in\RR$.  One might come up with the following payoff function to merge the two objectives:
\begin{align*}
&\innerprod{R_{t-1}}{p(z)}
+vY_{t-1}
+r_t\bigl(a_t(z)-x\bigr)
+v\bigl(x-1/2\bigr).
\end{align*}

The first term is the learner's regret against the adversarially chosen distribution $p(z)$ before round $t$, and
$vY_{t-1}$ prices its cumulative resource deviation against the pacing baseline. The difficulty is that these two
terms compare the learner to different reference points: regret is measured against the mixture $p(z)$, while
resource deviation is measured against the fixed baseline. We instead measure both quantities relative to the
same mixture $p(z)$. This alignment is crucial: if the learner chooses the spending level prescribed by $p(z)$,
then both the instantaneous regret and resource-deviation terms vanish simultaneously. We therefore use:
\eat{
The first term is the learner's regret against the adversarially chosen
distribution $p(z)$ before round $t$, and $vY_{t-1}(z)$ prices its cumulative
resource deviation against the pacing baseline.  The remaining two
terms are the regret and resource-deviation increments caused by the
current decision. This construction has some of the main ideas, but is insufficient for our analysis. We instead construct:}
\begin{align*}
&\innerprod{R_{t-1}}{p(z)}
+vE_{t-1}(z)
+r_t\bigl(a_t(z)-x\bigr)
+v\bigl(x-a_t(z)\bigr),
\end{align*}
where the dual variable $v$ penalizes the deviation from the cumulative spending of the mixture $p(z)$ instead of the cumulative spending of the pacing baseline. 

Finally, since the payoff function contains nonlinear terms involving the dual variables, we need to regularize both dual variables to apply Sion's minimax theorem \citep{sion1958general}. Define
\begin{equation}
\label{eq:sequential-minimax-payoff}
F_t(x,z,v)=
\innerprod{R_{t-1}}{p(z)}
+vE_{t-1}(z)
+r_t\bigl(a_t(z)-x\bigr)
+v\bigl(x-a_t(z)\bigr)
-\frac12v^2-z^2.
\end{equation}
Let $S_t=\sum_{s\le t} \left(\indicator{\text{$s$ is Red}} - \indicator{\text{$s$ is Blue}}\right)$. Then, $|S_t|\leq 1$, and we have
$
E_{t-1}(z)+x-a_t(z)
=
Y_{t-1}+x-\frac12-S_tz$, 
so the only nonlinear dependence on $(z,v)$ is
\[
-vS_tz-\frac12v^2-z^2.
\]
Its Hessian is
$$
\begin{pmatrix}
-2 & -S_t\\
-S_t & -1
\end{pmatrix},
$$
which is negative definite because $|S_t|\leq 1$.  Hence,
$F_t$ is convex in $x$ and jointly concave in $(z,v)$, and Sion's
minimax theorem applies.

Now, we are ready to define the algorithm, which simply solves a minmax problem in each round, shown in \Cref{algorithm:sequential-red-blue-balls}. It is worth noting that the algorithm is agnostic of the time horizon, and both the regret guarantee and the pacing guarantee do not depend on $T$.

\RestyleAlgo{ruled}
\begin{algorithm}[H]
\caption{Sequential online selection for alternating red and blue balls.}
\label{algorithm:sequential-red-blue-balls}
\DontPrintSemicolon
\For{$t=1,\dots, T$}{
    Observe the current reward $r_t$\;
    Define $F_t(x,z,v)=
\innerprod{R_{t-1}}{p(z)}
+vE_{t-1}(z)
+r_t\bigl(a_t(z)-x\bigr)
+v\bigl(x-a_t(z)\bigr)
-\frac12v^2-z^2$\;
    Choose
    \begin{align*}
        x_t\in
    &\arg\min_{x\in[0,1]}
    \max_{\substack{|z|\le1/2\\v\in\RR}} F_t(x,z,v)
    \end{align*}
}
\end{algorithm}

\begin{theorem}
    \Cref{algorithm:sequential-red-blue-balls} has constant regret and is constant-pacing with respect to the baseline $b=(1/2)_{t\in [T]}$, i.e.
\[\max\{R_t^R, R_t^B\}=O(1)\; \text{  and }\; 
\kol((x_t)_{t\in[T]},b)= O(1).\]
\end{theorem}

\begin{proof}
    Since $F_t(x,z,v)$ is affine in $x$ and jointly concave in $(z,v)$, Sion's minimax theorem gives 
    \begin{align*}
W_t=\min_{x\in[0,1]}\max_{\substack{|z|\le1/2\\v\in\RR}} F_t(x,z,v) = \max_{\substack{|z|\le1/2\\v\in\RR}} \min_{x\in[0,1]}
     F_t(x,z,v).
    \end{align*}
    
Define potential $\Phi_t(z,v) =  \innerprod{R_{t}}{p(z)}
+vE_{t}(z)
-\frac12v^2-z^2$. Thus $\Phi_t(z,v) = F(x_t, z,v)$ and we have 
\begin{align*}
   W_t = \max_{\substack{|z|\le1/2\\v\in\RR}} \Phi_t(z,v)
   = \max_{|z|\le1/2}\innerprod{R_t}{p(z)}+\frac12E_t(z)^2-z^2,
\end{align*}
by choosing the optimal $v^*=E_t(z)$.
We note that 
\begin{align*}
    F_t(x,z,v)
    &= 
\innerprod{R_{t-1}}{p(z)}
+vE_{t-1}(z)
+r_t\bigl(a_t(z)-x\bigr)
+v\bigl(x-a_t(z)\bigr)
-\frac12v^2-z^2.\\
    &=
\innerprod{R_{t-1}}{p(z)}
+vE_{t-1}(z)
+(r_t-v)\bigl(a_t(z)-x\bigr)
-\frac12v^2-z^2
\end{align*}

This shows that $F_t(x=a_t(z), z, v) = \innerprod{R_{t-1}}{p(z)}
+vE_{t-1}(z)
-\frac12v^2-z^2 = \Phi_{t-1}(z,v)$.
Thus we have
\begin{align*}
W_{t}=\max_{\substack{|z|\le1/2\\v\in\RR}} \min_{x\in[0,1]}F_t(x,z,v)
\leq \max_{\substack{|z|\le1/2\\v\in\RR}} F_t(x=a_t(z),z,v) = \max_{\substack{|z|\le1/2\\v\in\RR}} \Phi_{t-1}(z,v) = W_{t-1}.
\end{align*}
Since $W_0 = \max_{|z|\le1/2} -z^2 =0$, induction on $t$ gives $W_t\leq 0$. This means
\begin{align*}
\max_{|z|\le1/2}\innerprod{R_t}{p(z)}+\frac12E_t(z)^2-z^2\leq 0
\end{align*}
Evaluating at $z=\pm\frac12$ gives
$R_t^R+\frac12E_t(1/2)^2\le\frac14$
and
$R_t^B+\frac12E_t(-1/2)^2\le\frac14$.
Thus $R_t^R,R_t^B\le1/4$, proving the regret guarantee.

For the pacing guarantee, it suffices to show $|Y_t|=O(1), \forall t\in [T]$. Choose $(z_t,v_t)$ such that
$(x_t,z_t,v_t)$ is a saddle point of $F_t(x,z,v)$. Optimality in $v$ gives
$v_t=E_t(z_t)=Y_t-S_tz_t$, and hence $Y_t=v_t+S_tz_t$. Since
$x\mapsto F_t(x,z_t,v_t)$ is affine with slope $v_t-r_t$, optimality in $x$ implies that if $x_t=0$, then $v_t\ge r_t\ge0$; 
if $0<x_t<1$, then $v_t=r_t\in[0,1]$; and if $x_t=1$, then $v_t\le r_t\le1$.

We claim inductively that $-1/2\le Y_t\le3/2$. The claim holds for
$Y_0=0$.  If $x_t=0$, then $v_t\ge0$
gives $Y_t=v_t+S_tz_t\ge-1/2$ since $|z_t|\le1/2$, and $|S_t|\leq 1$, while
$Y_t=Y_{t-1}-1/2\le1$ by induction. If $0<x_t<1$, then
$Y_t=r_t+S_tz_t\in[-1/2,3/2]$, since $r_t\in[0,1]$. If $x_t=1$, then $v_t\le1$
gives $Y_t=v_t+S_tz_t\le3/2$, while
$Y_t=Y_{t-1}+1/2\ge0$ by induction. Thus we have
$-1/2\le Y_t\le3/2$ for all $t\in [T]$. If exact budget feasibility is desired, truncating the path once the budget $T/2$ is exhausted changes both the regret and pacing guarantees by only an additional $O(1)$.
\end{proof}

To solve the original problem with alternating $D$-length blocks of red balls and blue balls, we run \Cref{algorithm:sequential-red-blue-balls} for each queue. Since the algorithm for the subproblems has constant regret and is constant-pacing, it is easy to see that the resulting algorithm has at most $O(D)$ regret and is $O(D)$-pacing.

\section{Near-Optimal Algorithm for Learning with Pacing Experts}
\label{sec:near-optimal-algo}

The algorithm in \Cref{sec:warm-up-example} is in fact generalizable to any finite set of experts and multiple resources. We study both the ORA regime and the OLRC regime. Our algorithm for the first regime generalizes the toy algorithm in \Cref{sec:warm-up-example}, whereas the algorithm for the latter regime requires an additional stability analysis. We first assume the learner can play fractionally, i.e. deterministically playing a distribution $p_t\in \Delta(\actions)$ and get deterministic reward $r_t(p_t)$ and costs $c_t^i(p_t)$. We then remove this assumption and analyze sampled actions in \Cref{sec:sampled-feasibility}.

\subsection{Notation}

Throughout, an unqualified sum $\sum_t$ denotes the full-horizon sum $\sum_{t=1}^T$. Fix a baseline spending pattern for each resource $i\in [m]$: $b^i=(b_t^i)_{t\in[T]}\in[0,1]^T$ and $\sum_{t=1}^T b_t^i=B_i$. Every expert is initially assumed $D_i$-pacing in resource $i$:
\begin{equation}
S_t^{f,i}:=\sum_{s\le t}(c_s^{f,i}-b_s^i),\qquad |S_t^{f,i}|\le D_i\quad\forall f,i,t.
\label{eq:general-paced}
\end{equation}
Recall that a cost path $(c_t)_t$ is called $D$-subpacing with respect to a baseline $(b_t)_t$ if there exists a witness $d_t\in[0,b_t]$ such that
\begin{equation}
\left|\sum_{s\le t}(c_s-d_s)\right|\le D\qquad\forall t\in[T].
\label{eq:subpacing-definition}
\end{equation}
Thus, $D$-pacing is the special case witnessed by $d_t=b_t$.

Define
\begin{equation}\label{eq:eta}
    b_{\min}^i:=\textstyle \min_t b_t^i, \qquad \beta_i:=1/b_{\min}^i, \qquad \eta_i:=\min\{\alpha_i,\beta_i\} 
\end{equation}
where recall $\alpha_i$ is the reward-to-cost ratio of resource $i$ in \Cref{assumption:soft-conversion}.
In words, $b_{\min}^i$ denotes the smallest term in the baseline spending pattern $b_i$ and $\beta_i:=1/b_{\min}^i$ its reciprocal (with $\beta_i=+\infty$ if $b_{\min}^i=0$).
The quantity $\eta_i$ is the cheaper terminal price: additional rewards accrued by an expert after the algorithm exhausts some resource $i$ and hard stops at time $\tau$ can be priced either by the additional cost incurred by an expert at rate $\alpha_i$ or, when $b_{\min}^i>0$, by the remaining cost of the baseline at rate $\beta_i$, since $T-\tau \leq \beta_i\sum_{t>\tau}b_t^i$ and $\sum_{t>\tau}r_t^f\leq T-\tau$. Since $\alpha_i>0$ and $\beta_i\in(0,+\infty]$, every $\eta_i$ is strictly positive. 

We henceforth assume that $D_i>0$ for at least one resource $i$. If $D_i=0$ for every resource, then every paced expert has $c_t^{f,i}=b_t^i$ on every round (and every $0$-subpacing expert has $c_t^{f,i}\le b_t^i$), so the learning component reduces to the standard learning-with-expert-advice problem under the corresponding feedback model; we omit this degenerate case.

For any distribution over experts $w\in\Delta(\F)$, define 
\[
    \pi_t(w):=\textstyle \sum_fw_fq_t^f, \qquad r_t(w):=\textstyle \sum_fw_fr_t^f, \qquad c_t^i(w):=\textstyle \sum_fw_fc_t^{f,i}.
\]
That is, $\pi_t(w)$ is the probability distribution over actions induced by the distribution $w$ over experts, while $r_t(w)$ and $c_t^i(w)$ are the expected reward and expected cost in resource $i$, respectively, of choosing an action based on this distribution. (By a slight abuse of notation, we use $r_t(\cdot)$ and $c_t^i(\cdot)$ to denote the expected reward and cost for both distributions over actions and distributions over experts, with the intended domain clear from context.)

The algorithm uses a subroutine called a controller. 
At each time step, the controller chooses a distribution over experts $x_t\in\Delta(\F)$. The algorithm then plays the distribution over actions $p_t=\pi_t(x_t)$ (or $\Null$ if the controller has exceeded its budget). 

Define
\begin{equation}
R_t^f:=\sum_{s\le t}(r_s^f-r_s(p_s)),\qquad
Y_t^i:=\sum_{s\le t}(c_s^i(p_s)-b_s^i),\qquad
E_t^i(w):=Y_t^i-\sum_fw_fS_t^{f,i}.
\label{eq:states}
\end{equation}
$R^f_t$ denotes the controller's regret against expert $f$, $Y^i_t$ its ``overspending'' (which can be negative) compared to the baseline spending pattern, and $E^i_t$ the difference in overspending between the controller and the expected overspending of the distribution $w$ over experts: since $\sum_fw_f=1$, the baseline terms cancel and
$E_t^i(w)=\sum_{s\le t}(c_s^i(p_s)-c_s^i(w)).$ When $p_t=\pi_t(x_t)$, we also have $E_t^i(w)=\sum_{s\le t}(c_s^i(x_s)-c_s^i(w))$.
We reserve $R_t^f$ for the regret coordinate used in the potential.  In statements
of regret guarantees, we write
\begin{equation}
\Reg_t^f:=\sum_{s\le t}(r_s^f-r_s(p_s)).
\label{eq:regret-notation}
\end{equation}

\paragraph{Algorithmic resource prices.}
Let $\lambda=(\lambda_i)_{i\in[m]}\in\RR_{>0}^m$ be prices chosen by the algorithm. These are distinct from the structural conversion parameters $\alpha_i$ in \Cref{assumption:soft-conversion}. Write
\begin{equation}
D_{\lambda,\infty}:=\max_{i\in[m]}\lambda_iD_i,
\qquad \lambda_{\max}:=\max_{i\in[m]}\lambda_i.
\label{eq:D-lambda-infty}
\end{equation}
Let $v=(v_0,v_1,\ldots,v_m)\in\Delta_{m+1}$ be a distribution over the $m$ resources and a slack coordinate.
Fix $v\in\Delta_{m+1}, w\in \Delta(\experts)$. We say the resource penalty of coordinate $0$ is $0$, while the resource penalty of coordinate $i$ is $\lambda_iv_iE_t^i(w)$. Intuitively, this term penalizes deviation of the algorithm's spending from the spending of the expert mixture $w$, scaled by the chosen price $\lambda_i$ and the probability mass $v_i$.

\paragraph{Potential function.} We use the shifted negative entropies as regularizers for $w\in \Delta(\experts)$ and $v\in \Delta_{m+1}$:
\begin{align*}
    &\Psi_x(w):=\sum_fw_f\log w_f+\log F = D_{\rm KL}(w\| U_F)\\
    &\Psi_v(v):=\sum_{i=0}^mv_i\log v_i+\log(m+1) = D_{\rm KL}(v\| U_{m+1}).
\end{align*}
where $U_k$ denotes the uniform distribution over $k$ elements. Then $\Psi_x(w)\in [0, \log F]$ and $\Psi_v(v)\in [0, \log (m+1)]$.

Choose regularization coefficients $\rho_x,\rho_v>0$ to be defined later. We define the potential function and its maximum as
\begin{equation}
\Phi_t^\lambda(w,v):=\sum_fw_fR_t^f+\sum_{i=1}^m\lambda_iv_iE_t^i(w)
-\rho_x\Psi_x(w)-\rho_v\Psi_v(v),
\qquad
M_t^\lambda:=\max_{w,v}\Phi_t^\lambda(w,v),
\label{eq:potential-lambda}
\end{equation}
Here and throughout, the maximum is taken over $(w,v)\in \Delta(\F)\times\Delta_{m+1}$, and we omit this domain when it is clear from context. The maximum potential $M_t^\lambda$ provides a useful upper bound on the combined regret and resource penalty against an arbitrary expert mixture $w$ and resource distribution $v$.\footnote{Note that $\Phi$ strongly resembles the potential function from the alternating red and blue ball problem in Section~\ref{section:opt-red-blue}.
The only difference is that we have opted for entropy over quadratic regularization, a standard technique to optimize the regret guarantee's dependence on the number of experts (and in this case, number of resources as well).
The purpose of $\lambda_i$ is to control the loss of reward on rounds where the algorithm must choose the null action (due to overspending), contrary to the controller's recommendation.
For the core algorithm, we choose $\lambda_ i = \eta_i$, while for the sub-pacing algorithm in \Cref{subsec:extension-subpacing-algo}, we choose $\lambda_i = \alpha_i$.}

To define how the potential function changes after a new decision $x\in \Delta(\experts)$, we let
\begin{equation*}
    F_t^\lambda(x,w,v):=\Phi_{t-1}^\lambda(w,v)+r_t(w)-r_t(x)
+\sum_{i=1}^{m}\lambda_iv_i\bigl(c_t^i(x)-c_t^i(w)\bigr).
\end{equation*}

\paragraph{Hard stop.}
Let $\textstyle C_{t-1}^i:=\sum_{s\leq t-1}c_s^i(p_s)$
denote the cumulative spending of the controller.
Before round $t$, if the hard stop has already taken effect, or if $C_{t-1}^i+1>B_i$ for some $i$, output $\Null$ from round $t$ onward. 
Otherwise, run the controllers outlined in Section~\ref{sec:core-algorithms}. 
Since every one-round cost is at most 1, this guarantees fractional budget feasibility $\sum_t c_t^i(p_t)\le B_i$.

\subsection{Key Curvature Lemma}
The potential function involves three nonlinear terms: $-\rho_x\Psi_x(w)$, $-\rho_v\Psi_v(v)$, and $\sum_{i=1}^{m}\lambda_iv_iE_t^i(w)$.  The first two terms are shifted entropy terms that provide strong concavity in $w$ and $v$, respectively, in $L^1$-norm. The third term is problematic because it involves products of $v$ and $w$ and the nonlinearity can erase the concavity in $w$ and $v$. However, with appropriately chosen parameters $\rho_x, \rho_v$, we show that the potential function satisfies the following concavity properties.
\begin{lemma}
\label{lem:curvature-lambda}
Assume
\begin{equation}
\rho_x\rho_v\ge4D_{\lambda,\infty}^2.
\label{eq:entropic-core-curvature-condition}
\end{equation}
Then for every $t$, $\Phi_t^\lambda$ is jointly concave in $(w,v)$, and $F_t^\lambda(x,\cdot,\cdot)$ is jointly concave in $(w,v)$ for every fixed $x\in \Delta(\experts)$. 

Moreover, if $(w^*,v^*)$ maximizes $\Phi_t^\lambda$, then both $w^*$ and $v^*$ have full support and, for every feasible $(w,v)$,
\begin{equation}
\Phi_t^\lambda(w,v)
\le M_t^\lambda
-\frac{\rho_x}{2}D_{\rm KL}(w\|w^*)
-\frac{\rho_v}{2}D_{\rm KL}(v\|v^*).
\label{eq:quant-curvature-lambda}
\end{equation}
\end{lemma}

\begin{proof}
Consider interior points $(w,v)$ and $(w',v')=(w+\Delta w, v+\Delta v)$. To show $\Phi_t^\lambda$ is jointly concave in $(w,v)$, it suffices to show \[\Phi_t^\lambda(w+\Delta w, v+\Delta v)-\Phi_t^\lambda(w, v)\leq \innerprod{\nabla \Phi_t^\lambda(w,v)}{(\Delta w, \Delta v)}.\] We have
 \begin{align*}
    &\Phi_t^\lambda(w+\Delta w, v+\Delta v)-\Phi_t^\lambda(w, v)\\
    &= \langle \Delta w, R_t\rangle + \sum_{i=1}^m \lambda_i \Delta v_i Y_t^i - \sum_{i=1}^m \lambda_i \Delta v_i \langle w, S_t^i\rangle - \sum_{i=1}^{m}\lambda_i v_i\langle \Delta w, S_t^i\rangle \\
    &\quad - \sum_{i=1}^m \lambda_i \Delta v_i \langle \Delta w, S_t^i\rangle
    - \rho_x (\Psi_x(w+\Delta w)-\Psi_x(w))
    - \rho_v (\Psi_v(v+\Delta v)-\Psi(v))\\
    &=\innerprod{\nabla \Phi_t^\lambda(w,v)}{(\Delta w, \Delta v)}  - \sum_{i=1}^m \lambda_i \Delta v_i \langle \Delta w, S_t^i\rangle
    - \rho_x D_{\rm KL}(w+\Delta w\| w)
    - \rho_v D_{\rm KL}(v+\Delta v\| v).\\
 \end{align*}
 where we use the fact that 
 $\Psi_x(w+\Delta w)-\Psi_x(w) = \nabla\Psi_x(w)\Delta w + D_{\rm KL}(w+\Delta w\| w)$ and $\Psi_v(v+\Delta v)-\Psi_v(v) = \nabla\Psi_v(v)\Delta v + D_{\rm KL}(v+\Delta v\| v)$.

We now bound the first nonlinear term, which crucially relies on the pacing property of the experts:
\begin{align*}
 \left|\sum_{i=1}^m\lambda_i\Delta v_i\innerprod{\Delta w}{S_t^i}\right|
&\leq D_{\lambda,\infty}\|\Delta v\|_1\|\Delta w\|_1 \tag{$\|S_t^i\|_\infty\le D_i$}\\
&\leq \frac{\rho_x}{4}\|\Delta w\|_1^2
   +\frac{D_{\lambda,\infty}^2}{\rho_x}\|\Delta v\|_1^2 \tag{$\frac{\sqrt{\rho}_x \|\Delta w\|_1}{2} -\frac{D_{\lambda,\infty} \|\Delta v\|_1}{\sqrt{\rho_x}})^2\geq 0$}\\
&\le \frac{\rho_x}{2}D_{\rm KL}(w'\|w)
   +\frac{2D_{\lambda,\infty}^2}{\rho_x}D_{\rm KL}(v'\|v)\tag{Pinsker's Inequality: $D_{\rm KL}(P\|Q)\geq \frac{\|P-Q\|^2_1}{2}$}\\
&\le \frac{\rho_x}{2}D_{\rm KL}(w'\|w)
   +\frac{\rho_v}{2}D_{\rm KL}(v'\|v)\tag{$\rho_x\rho_v\ge 4D_{\lambda,\infty}^2$}
\end{align*}

This implies 
 \begin{align}
 \label{eq:concavity-KL}
\Phi_t^\lambda(w', v')-\Phi_t^\lambda(w, v)
\leq \innerprod{\nabla \Phi_t^\lambda(w,v)}{(\Delta w, \Delta v)} - \frac{\rho_x}{2} D_{\rm KL}(w'\| w)
    - \frac{\rho_v}{2} D_{\rm KL}(v'\| v).
\end{align}

Thus joint concavity of $\Phi_t^\lambda$ follows because $D_{\rm KL}(w'\| w)\geq 0$ and $D_{\rm KL}(v'\| v)\geq 0$. The joint concavity of $F_t^\lambda(x,\cdot,\cdot)$ can be derived analogously because the nonlinear terms are the same in both cases.

Now, when $(w^*,v^*)$ maximizes $\Phi_t^\lambda$, $w^*$ and $v^*$ must have full support: consider $w^*$ and suppose it has a zero coordinate. Then shifting $\varepsilon$ mass from a positive coordinate to a 0 coordinate of $w^*$ changes the objective by $\rho_x\varepsilon \log (1/\varepsilon) + O(\varepsilon)>0$, for $\varepsilon$ sufficiently small, which contradicts the optimality of $w^*$. The argument for $v^*$ is analogous. Additionally, $\innerprod{\nabla \Phi_t^\lambda(w^*,v^*)}{(w-w^*, v-v^*)}\leq 0, \forall (w,v)\in \Delta(\experts)\times \Delta_{m+1}$ by first-order necessary optimality condition for $\Phi_t^\lambda$ over the convex domain.

Thus by \eqref{eq:concavity-KL}, we have for every feasible $(w,v)$
\begin{align*}
\Phi_t^\lambda(w, v)-\Phi_t^\lambda(w^*, v^*)
&\leq 
\innerprod{\nabla \Phi_t^\lambda(w^*,v^*)}{(w-w^*, v-v^*)} 
- \frac{\rho_x}{2} D_{\rm KL}(w\| w^*)
- \frac{\rho_v}{2} D_{\rm KL}(v\| v^*)\\
&\leq - \frac{\rho_x}{2} D_{\rm KL}(w\| w^*)- \frac{\rho_v}{2} D_{\rm KL}(v\| v^*)\\
\end{align*}
Substituting $M_t^\lambda = \Phi_t^\lambda(w^*, v^*)$ finishes the proof.
\end{proof}

\subsection{Core Algorithms}\label{sec:core-algorithms}

In this section, we give the core algorithms for both the ORA and OLRC regimes.
These algorithms guarantee optimal regret bounds against $D$-pacing experts without exceeding the budgets $B$, where $D=(D_i)_{i\in [m]}$ and $B=(B_i)_{i\in [m]}$.
Section~\ref{sec:extensions} then describes how to extend the algorithm to compete against $D$-subpacing experts, to be subpacing itself, and to convert the fractional guarantees to sampled play with pathwise budget feasibility and high-probability regret guarantees.

\RestyleAlgo{ruled}
\begin{algorithm}[ht]
\caption{Core ORA Algorithm}
\label{algorithm:core-ORA}
\DontPrintSemicolon
Let $\lambda_i = \eta_i = \min\{\alpha_i, \beta_i\}$ where
\begin{itemize}
    \item $\alpha_i$ is the reward-to-cost ratio of resource $i$
    \item $\beta_i = 1/\min_t b^i_t$ is the reciprocal of the smallest term in the baseline spending pattern $b^i$\;
\end{itemize}
\For{$t=1,\dots, T$}{
    If $C^i_{t-1} = \sum_{s<t} c^i_s(p_s) > B_i - 1$ for some resource $i$, break and play $\perp$ from round $t$ onward\;
    Otherwise, observe the current round's rewards vector $r_t$ and cost vectors $(c^i_t)_{i \in [m]}$\;
    Define 
    \[
        \textstyle F_t^\lambda(x,w,v):=\Phi_{t-1}^\lambda(w,v)+r_t(w)-r_t(x) +\sum_{i=1}^{m}\lambda_iv_i\bigl(c_t^i(x)-c_t^i(w)\bigr)
    \] 
    where $\Phi_{t-1}^\lambda(w,v):=\sum_fw_fR_{t-1}^f+\sum_{i=1}^m\lambda_iv_iE_{t-1}^i(w) - \rho_x\Psi_x(w)-\rho_v\Psi_v(v)$\;
    Choose
    \[
        \textstyle x_t \in \arg\min_{x\in\Delta(\F)}\max_{w \in \Delta(\mathcal{F)}, v \in \Delta_{m+1}} F_t^\lambda(x,w,v)\;
    \]
    Play $p_t = \pi_t(x_t)$ and receive fractional reward $r_t(p_t)$ and fractional costs $c_t^i(p_t)$\;
}
\end{algorithm}

\paragraph{With pre-decision feedback.}

In the ORA regime, the core algorithm, stated in Algorithm~\ref{algorithm:core-ORA}, follows the recommendations of the following controller, as long as it is not about to exceed any budget $B_i$: in each round, after observing the current reward and cost vectors, the controller recommends
\begin{equation}
x_t\in\arg\min_{x\in\Delta(\F)}\;\max_{w \in \Delta(\mathcal{F)}, v \in \Delta_{m+1}}F_t^\lambda(x,w,v),
\label{eq:with-predecision-lambda}
\end{equation}
and the algorithm plays $p_t = \pi_t(x_t)$ (or $\perp$ if $C^i_{t-1} > B_i - 1$ for some resource $i$). 

Intuitively, $F_t^{\lambda}$ is the hypothetical potential function at $t$ if the algorithm were to play $x$ in round $t$, and the controller recommends the $x_t$ that minimizes the potential against the worst expert mixture and resource distribution. This generalizes the decision rule for the alternating red and blue ball problem in Section~\ref{section:opt-red-blue}.

Under the regularization condition in \Cref{lem:curvature-lambda}, this controller guarantees that the potential does not increase over time.

\begin{lemma}
\label{lem:progress-with-feedback}
Under \eqref{eq:entropic-core-curvature-condition}, before the hard stop the controller with pre-decision feedback satisfies \[M_t^\lambda \leq M_{t-1}^\lambda. \]
\end{lemma}

\begin{proof}
For fixed $(w,v)$, the potential $F_t^\lambda$ is affine in $x$, while \Cref{lem:curvature-lambda} shows that, under its regularization condition, it is jointly concave in $(w,v)$ for fixed $x$. Hence, Sion's minimax theorem applies. Moreover, $F_t^\lambda(w,w,v)=\Phi_{t-1}^\lambda(w,v)$ and $\Phi_t^\lambda(w,v)=F_t^\lambda(x_t,w,v)$, so
\begin{align*}
M_t^\lambda 
&= \max_{w,v}\Phi_t^\lambda(w,v)
= \max_{w,v}F_t^\lambda(x_t,w,v)\\
&=\min_x\max_{w,v}F_t^\lambda(x,w,v)
=\max_{w,v}\min_xF_t^\lambda(x,w,v)\tag{Sion's Minimax Theorem}\\
&\le \max_{w,v} F_t^\lambda(w,w,v)
= \max_{w,v}\Phi_{t-1}^\lambda(w,v)
= M_{t-1}^\lambda.
\label{eq:with-predecision-progress-lambda}
\end{align*}  
\end{proof}

\RestyleAlgo{ruled}
\begin{algorithm}[ht]
\caption{Core OLRC Algorithm}
\label{algorithm:core-OLRC}
\DontPrintSemicolon
Let $\lambda_i = \eta_i = \min\{\alpha_i, \beta_i\}$ where
\begin{itemize}
    \item $\alpha_i$ is the reward-to-cost ratio of resource $i$
    \item $\beta_i = 1/\min_t b^i_t$ is the reciprocal of the smallest term in the baseline spending pattern $b^i$\;
\end{itemize}
\For{$t=1,\dots, T$}{
    If $C^i_{t-1} = \sum_{s<t} c^i_s(p_s) > B_i - 1$ for some resource $i$, break and play $\perp$ from round $t$ onward\;
    Otherwise, define 
    \[
        \textstyle \Phi_{t-1}^\lambda(w,v):=\sum_fw_fR_{t-1}^f+\sum_{i=1}^m\lambda_iv_iE_{t-1}^i(w) - \rho_x\Psi_x(w)-\rho_v\Psi_v(v)\;
    \] 
    Let
    \[
        \textstyle (x_t,v_t)\in\arg\max_{w \in \Delta(\mathcal{F}), v \in \Delta_{m+1}}\Phi_{t-1}^\lambda(w,v)\;
    \]
    Play $p_t = \pi_t(x_t)$ and receive fractional reward $r_t(p_t)$ and fractional costs $c_t^i(p_t)$\;
    Observe reward vector $r_t$ and cost vectors $(c^i_t)_{i \in [m]}$
}
\end{algorithm}

\paragraph{Without pre-decision feedback.}

In this regime, the controller does not have access to the current round's rewards and costs in the OLRC regime, so it cannot recommend the distribution $x$ over experts that minimizes the worst-case potential $\max_{w,v} F^\lambda_t (x, w, v)$.
Instead, the controller simply recommends the output of Follow-the-Regularized-Leader (FTRL) on the regularized objective $\max_v \Phi^\lambda_{t-1}(\cdot, v)$.
As before, the algorithm, stated in Algorithm~\ref{algorithm:core-OLRC}, follows the controller's recommendation, as long as it is not about to exceed any budget $B_i$. More specifically, the controller chooses
\begin{equation}
(x_t,v_t)\in\arg\max_{w,v}\Phi_{t-1}^\lambda(w,v),
\label{eq:without-predecision-lambda}
\end{equation}
and the algorithm plays $p_t=\pi_t(x_t)$ (or $\perp$ if $C^i_{t-1} > B_i - 1$ for some resource $i$).
Similar to how the stability of FTRL bounds its regret in the standard online learning setting, we will see that the stability of the controller bounds the increments of the potential $M^\lambda_t$.

\begin{lemma}
\label{lem:progress-without-feedback}
Under \eqref{eq:entropic-core-curvature-condition}, before the hard stop the controller without pre-decision feedback satisfies
\[
M_t^\lambda-M_{t-1}^\lambda\le\frac{(1+\lambda_{\max})^2}{\rho_x}.
\]
\end{lemma}
\begin{proof}
Before the hard stop, the additional algorithm reward in round $t$ is $r_t(x_t)$ and the additional overspending with respect to expert mixture $w$ is
$E_t^i(w)-E_{t-1}^i(w)=c_t^i(x_t)-c_t^i(w) = \innerprod{c_t^i}{x_t-w}$. Hence
\[
\Phi_t^\lambda(w,v)-\Phi_{t-1}^\lambda(w,v)
=\innerprod{r_t-\sum_{i=1}^m\lambda_iv_ic_t^{i}}{w-x_t}
\le(1+\lambda_{\max})\|w-x_t\|_1.
\]
Since $(x_t,v_t)$ maximizes $\Phi_{t-1}^\lambda$, we have
\begin{align*}
\Phi_t^\lambda(w,v)-M_{t-1}^\lambda
&= \Phi_t^\lambda(w,v)-\Phi_{t-1}^\lambda(w,v) + \Phi_{t-1}^\lambda(w,v)-M_{t-1}^\lambda\\
&\leq (1+\lambda_{\max})\|w-x_t\|_1 - \frac{\rho_x}{2} D_{\rm KL}(w\| x_t)- \frac{\rho_v}{2} D_{\rm KL}(v\| v_t) \tag{by \Cref{lem:curvature-lambda}}\\ 
&\leq (1+\lambda_{\max})\|w-x_t\|_1 - \frac{\rho_x}{4} \|w-x_t\|^2_1\tag{Pinsker's Inequality}\\
&\leq \frac{(1+\lambda_{\max})^2}{\rho_x}
\end{align*}
where the last line follows by optimizing the quadratic function in $\|w-x_t\|_1$.
\end{proof}

\subsection{Completing the Analysis}
\label{sec:complete-analysis}

We first derive a lemma for bounding the rewards and costs of an expert after the algorithm hard stops. Let $A_t^{f,j}:=\sum_{s\le t}c_s^{f,j}$ be cumulative spending of expert $f$ in resource $j$ by the end of round $t$ and
$B_t^j:=\sum_{s\le t}b_s^j$ be the cumulative baseline spending in resource $j$ by the end of round $t$. Recall that $\textstyle C_{t}^j=\sum_{s\leq t}c_s^j(p_s)$ is the analogous term for algorithm spending.
\begin{lemma}
\label{lem:tail}
Suppose that by the end of round $\tau$, there exists some resource $j$ such that $C_\tau^j\geq B_j-h_j$ for some $h_j\geq 0$.
Then
\begin{align}
\sum_{t>\tau}c_t^{f,j}&\le E_\tau^j(e_f)+h_j+D_j,\label{eq:cost-tail}\\
\sum_{t>\tau}r_t^f&\le\alpha_j\bigl(E_\tau^j(e_f)+h_j+D_j\bigr)+\sum_{t>\tau}\varepsilon_t^j,\label{eq:soft-tail}\\
\text{If $\beta_j<\infty$, }
\sum_{t>\tau}r_t^f&\le\beta_j\bigl(E_\tau^j(e_f)+h_j+D_j\bigr).\label{eq:temporal-tail}
\end{align}
\end{lemma}

\begin{proof}
Since each expert is $D$-pacing with respect to the baselines, we have $A_T^{f,j}\le B_j+D_j$. Additionally, by definition, $E_\tau^j(e_f)= C_\tau^j-A_\tau^{f,j}$. Thus
\begin{align*}
\sum_{t>\tau}c_t^{f,j}
= A_T^{f,j} - A_\tau^{f,j}
&\le  B_j+D_j - (C_\tau^j - E_\tau^j(e_f))\\
&\leq (B_j-C_\tau^j) + E_\tau^j(e_f) + D_j\\
&\le h_j+E_\tau^j(e_f)+D_j.
\end{align*}

By \Cref{assumption:soft-conversion}, we have $r_t^f\le\alpha_jc_t^{f,j}+\varepsilon_t^j$. Summing this over $t$ together with \eqref{eq:cost-tail} gives \eqref{eq:soft-tail}.

Additionally, we have
\begin{align*}
\sum_{t>\tau}b_t^j &= B_j - B_\tau^j \tag{$\sum_{t=1}^T b_t^j=B_j$}\\
&= (B_j-C_\tau^j)+(C_\tau^j-B_\tau^j)\\
&\leq h_j+(C_\tau^j-A_\tau^{f,j}) + (A_\tau^{f,j}-B_\tau^j) \\
&\le h_j+E_\tau^j(e_f)+D_j \tag{$A_\tau^{f,j}-B_\tau^j\le D_j$}.
\end{align*}
If $\beta_j<\infty$, then $\sum_{t>\tau}r_t^f\le T-\tau\le\beta_j\sum_{t>\tau}b_t^j$, this gives \eqref{eq:temporal-tail}.
\end{proof}

\begin{theorem}
\label{thm:general-lambda}
Assume $\lambda_i\ge\eta_i$ for every resource and choose $\rho_x,\rho_v>0$ satisfying
\eqref{eq:entropic-core-curvature-condition}. Then both algorithms are budget feasible via fractional play. Moreover, for every expert $f$ simultaneously,
\begin{align}
\Reg_T^f&\le\rho_x\log F+\rho_v\log(m+1)+\mathsf{Tail},
&&\text{with pre-decision feedback},\label{eq:general-lambda-with-pre}\\
\Reg_T^f&\le\rho_x\log F+\rho_v\log(m+1)
+\frac{T(1+\lambda_{\max})^2}{\rho_x}+\mathsf{Tail},
&&\text{without pre-decision feedback},\label{eq:general-lambda-without-pre}
\end{align}
where
\begin{equation}
\mathsf{Tail}:=
\max_i\left\{\eta_i(1+D_i)+\indicator{\alpha_i<\beta_i}V_i\right\}.
\label{eq:tail-lambda}
\end{equation}
\end{theorem}
\begin{proof}
The stopping rule gives fractional budget feasibility. Let
\[
\mathsf{Stab}_s:=
\begin{cases}
0,&\text{with pre-decision feedback},\\[1mm]
\displaystyle\sum_{t\le s}\frac{(1+\lambda_{\max})^2}{\rho_x},&\text{without pre-decision feedback}.
\end{cases}
\]
Since $M_0^\lambda=0$, \Cref{lem:progress-with-feedback} and \Cref{lem:progress-without-feedback} imply
$M_s^\lambda\le\mathsf{Stab}_s$ for every round $s$ before the algorithm hard stops.

Recall the potential function at $T$ is $\Phi_T^\lambda(w,v):=\sum_fw_fR_T^f+\sum_{i=1}^m\lambda_iv_iE_T^i(w)
-\rho_x\Psi_x(w)-\rho_v\Psi_v(v)$.
If no hard stop occurs, evaluating the potential function at $(e_f,e_0)$ gives
\[
R_T^f+0-\rho_x\log F-\rho_v\log(m+1)\le \mathsf{Stab}_T,
\]
which gives the desired bounds.

Otherwise let resource $j$ trigger the hard stop at the end of round $\tau$. Then $B_j-C_\tau^j<1$. Set $s_j:=\eta_j/\lambda_j\in[0,1]$ and evaluate the prefix potential at $\bigl(e_f,(1-s_j)e_0+s_je_j\bigr)$. The resource term is $\lambda_js_jE_\tau^j(e_f)=\eta_jE_\tau^j(e_f)$, while
$\Psi_x(e_f)=\log F$ and $\Psi_v((1-s_j)e_0+s_je_j)\le\log(m+1)$. Therefore
\[
R_\tau^f + \eta_jE_\tau^j(e_f) -\rho_x\log F - \rho_v\log(m+1)
\le \mathsf{Stab}_\tau
.
\]

Since $B_j-C_\tau^j<1$, \Cref{lem:tail} applies with $h_j=1$. If $\alpha_j<\beta_j$, use \eqref{eq:soft-tail}; then $\eta_j=\alpha_j$ and the suffix residual is at most $V_j$. If $\alpha_j\ge\beta_j$, use \eqref{eq:temporal-tail}; then $\eta_j=\beta_j$. Thus in both cases
\[
\sum_{t>\tau}r_t^f
\le \eta_j\bigl(E_\tau^j(e_f)+1+D_j\bigr)
+\indicator{\alpha_j<\beta_j}V_j\leq \eta_jE_\tau^j(e_f) + \mathsf{Tail}.
\]
Thus we have 
\begin{align*}
\Reg_T^f\leq R_\tau^f + \sum_{t>\tau}r_t^f\leq 
\mathsf{Stab}_\tau +\rho_x\log F + \rho_v\log(m+1)+\mathsf{Tail}
\end{align*}
\end{proof}

For the core algorithm, we use $\lambda_i=\eta_i=\min\{\alpha_i,\beta_i\}$.
Define
\[
D_{\eta,\infty}:=\max_i\eta_iD_i,
\qquad \eta_{\max}:=\max_i\eta_i.
\]

\begin{theorem}[Main core guarantee]
\label{thm:core}
Choose $\lambda_i=\eta_i, \forall i\in [m]$. Let $\rho_x,\rho_v>0$ satisfy
$\rho_x\rho_v\ge4D_{\eta,\infty}^2$. Then both algorithms are budget feasible via fractional play, and for every expert $f$ simultaneously,
\begin{align}
\Reg_T^f&\le \rho_x\log F+\rho_v\log(m+1)+\mathsf{Tail},
&&\text{with pre-decision feedback},\label{eq:core-with-pre-eta}\\
\Reg_T^f&\le \rho_x\log F+\rho_v\log(m+1)
+\frac{T(1+\eta_{\max})^2}{\rho_x}+\mathsf{Tail},
&&\text{without pre-decision feedback},\label{eq:core-without-pre-eta}
\end{align}
where
\[
\mathsf{Tail}
:=\max_i\left\{\eta_i(1+D_i)+\indicator{\alpha_i<\beta_i}V_i\right\}.
\]
Optimizing the regularization parameters gives
\begin{align}
\Reg_T^f&\le \mathsf{Tail}
+O\left(D_{\eta,\infty}\sqrt{\log F\,\log(m+1)}\right),
&&\text{with pre-decision feedback},\label{eq:opt-with-pre-eta}\\
\Reg_T^f&\le \mathsf{Tail}
+O\left((1+\eta_{\max})\sqrt{T\log F}
+D_{\eta,\infty}\sqrt{\log F\,\log(m+1)}\right),
&&\text{without pre-decision feedback}.\label{eq:opt-without-pre-eta}
\end{align}
\end{theorem}
\begin{proof}
The untuned regret guarantees follow from \Cref{thm:general-lambda} with $\lambda=\eta$. For the optimized guarantees, take
$\rho_v=4D_{\eta,\infty}^2/\rho_x$. By the standing assumption that $D_i>0$ for some $i$ and the positivity of every $\eta_i$, one has $D_{\eta,\infty}>0$. With pre-decision feedback, minimizing
\[
\rho_x\log F+\frac{4D_{\eta,\infty}^2\log(m+1)}{\rho_x}
\]
gives $O\left(D_{\eta,\infty}\sqrt{\log F\log(m+1)}\right)$. Without pre-decision feedback, minimizing
\[
\rho_x\log F
+\frac{4D_{\eta,\infty}^2\log(m+1)+T(1+\eta_{\max})^2}{\rho_x},
\]
and using $\sqrt{a+b}\le\sqrt a+\sqrt b$ give \eqref{eq:opt-without-pre-eta}.
\end{proof}

\section{Extensions}
\label{sec:extensions}

The following three extensions are modular and may be composed in any combination.  The first enlarges the comparator class from pacing to subpacing experts, the second additionally requires the learner itself to be subpacing, and the third converts any resulting fractional controller to sampled play with pathwise budget feasibility and high-probability guarantees, under full-information.  We first give a characterization of subpacing used by the first two extensions.

The key part of the characterization is that if a spending pattern is $D$-subpacing with respect to a baseline, then it cannot overspend the baseline by more than $D$ on any prefix, or by more than $2D$ over any interval. This allows us to convert a subpacing spending pattern into a pacing one by artificially raising its spending when it falls too far below the baseline, without causing the resulting path to exceed the baseline by more than $D$ on any prefix. 

Additionally, we show a downward closure property of the set of $D$-subpacing spending patterns.

\begin{lemma}[Subpacing characterization and downward closure]
\label{lem:subpace-char}
Let $A_t:=\sum_{s\le t}c_s$ and $B_t:=\sum_{s\le t}b_s$, with $A_0=B_0=0$.  The path $c$ is $D$-subpacing with respect to $b$ if and only if
\begin{equation}
A_t-B_t\le D\quad\forall t,
\qquad
(A_t-A_s)-(B_t-B_s)\le2D\quad\forall 0\le s<t\le T.
\label{eq:subpace-char}
\end{equation}
Consequently, if $0\le c_t\le\bar c_t$ pointwise and $\bar c$ is $D$-subpacing with respect to $b$, then $c$ is also $D$-subpacing.
\end{lemma}
\begin{proof}
Suppose first that $c$ is $D$-subpacing with respect to $b$ and $d_t\in[0,b_t]$ is a subpacing witness. Let $X_t:=\sum_{s\le t}d_s$.  Since $X_t\le B_t$ and $|A_t-X_t|\le D$,
$A_t-B_t\le A_t-X_t\le D$.  For $0\le s<t$,
\begin{align*}
(A_t-A_s)-(B_t-B_s)
&\le (A_t-A_s)-(X_t-X_s)\\
&=(A_t-X_t)-(A_s-X_s)\le2D,
\end{align*}
which proves necessity.
The first inequality follows from the fact that
\[
    X_t - X_s = \sum_{\tau = s+1}^t d_\tau \leq \sum_{\tau = s+1}^t b_\tau = B_t - B_s.
\]

Conversely, assume \eqref{eq:subpace-char}.  Let
\[K_t:=\max_{0\le s\le t}(B_s-A_s-D)_+\] be the maximum prefix underspending of $c$ with respect to the baseline beyond $D$ by the end of round $t$. Define $X_t:=B_t-K_t$ and $d_t:=X_t-X_{t-1}$. We will show that $d$ is a subpacing witness of $c$. 

Since $K_t\geq 0$, $X_t\le B_t$.  We claim that $|A_t-X_t|\le D$.  For the lower bound, $K_t\ge(B_t-A_t-D)_+$ implies $A_t-B_t+K_t\ge-D$, i.e. $A_t-X_t\ge-D$.  For the upper bound, if $K_t=0$, then $A_t-X_t=A_t-B_t\le D$.  If $K_t>0$, choose $s\le t$ attaining the maximum, so $K_t=B_s-A_s-D$.  If $s=t$, then $A_t-X_t=-D$.  If $s<t$, the second inequality in \eqref{eq:subpace-char} gives
\[
A_t-X_t=A_t-B_t+B_s-A_s-D
=(A_t-A_s)-(B_t-B_s)-D\le D.
\]
Thus, $|A_t-X_t|\le D$ for every $t$.

It remains to show that $d_t=X_t-X_{t-1}$ lies in $[0,b_t]$.  Since $K_t$ is nondecreasing,
$d_t=b_t-(K_t-K_{t-1})\le b_t$.  If $K_t=K_{t-1}$, then $d_t=b_t\ge0$.  Otherwise the new maximum is attained at $t$, so $K_t=B_t-A_t-D$ and
$K_{t-1}\ge B_{t-1}-A_{t-1}-D$.  Hence
$K_t-K_{t-1}\le b_t-c_t$, which implies $d_t\ge c_t\ge0$.  Therefore $(d_t)$ is a valid witness and $c$ is $D$-subpacing.

For downward closure, let $\bar A_t:=\sum_{s\le t}\bar c_s$.  Pointwise domination implies
$A_t-B_t\le\bar A_t-B_t$ and, for every $s<t$,
$(A_t-A_s)-(B_t-B_s)\le(\bar A_t-\bar A_s)-(B_t-B_s)$.  Thus if $\bar c$ satisfies \eqref{eq:subpace-char}, then so does $c$, and the characterization proves the claim.
\end{proof}

\subsection{Extension 1: Subpacing Experts}
\label{subsec:subpacing-experts}

Note that the core algorithm cannot handle subpacing experts because these experts may underspend too much and cause the potential $\Phi^\lambda_t$ to become non-concave in $(w,v)$.
To extend the core algorithm to handle subpacing experts, we artificially raise the spending of these experts so that their cumulative (artificial) spending in each resource $i$ is never $D_i$ below the baseline's spending in this resource.
We then feed each expert's artificial costs into the core algorithms instead of their true costs.

At a high level, we inflate an expert's spending in resource $i$ in the following way.
The first time $t$ an expert $f$ spends less than the baseline by $D_i$, that is,
\[
    S_t^{f,i} = \sum_{\tau=1}^t (c^{f,i}_\tau - b^i_\tau) < - D_i \quad \text{and} \quad S_{s}^{f,i} = \sum_{\tau=1}^{s} (c^{f,i}_\tau - b^i_\tau) \geq - D_i \,\forall\, s < t,
\]
we increase the costs faced by that expert so that her spending relative to the baseline does not decrease further and remains $-D_i$ away from the baseline's spending.
That is, we inflate her costs so that her (artificial) spending is
\[
    \bar c_t^{f,i} \coloneqq b_t^i - S_{t-1}^{f,i}  - D_i.
\]
Note that
\[
    S_{t-1}^{f,i} + (\bar c_t^{f,i} - b_t^i) = -D_i.
\]
Let 
\[
    \bar S_t^{f,i} \coloneqq S_{t-1}^{f,i} + (\bar c_t^{f,i} - b_t^i).
\]
If the expert continues to underspend in the next round, that is, $c_{t+1}^{f,i} < b^i_{t+1}$, then we continue to inflate her costs so that her (artificial) spending solves
\[
    \bar S_t^{f,i} + (\bar c_{t+1}^{f,i} - b_{t+1}^i) = -D_i.
\]
Equivalently, $\bar c_{t+1}^{f,i} \coloneqq b_{t+1}^i$ since $\bar S_t^{f,i} = -D_i$.
Otherwise, $c_{t+1}^{f,i} \geq b^i_{t+1}$, and $\bar c_{t+1}^{f,i} \coloneqq c_{t+1}^{f,i}$, that is, we stop inflating the costs faced by the expert. 

If, in a later round $t' > t$, the expert's (artificial) spending again falls below the baseline's spending by more that $D_i$, that is,
\[
    \bar S_{t'-1}^{f,i} + (c_{t'}^{f,i} - b_{t'}^{i}) < - D_i \quad \text{but} \quad \bar S_{s}^{f,i} \geq - D_i \,\forall\, s < t',
\]
we once again artificially raise costs so that the difference $\bar S_{t'}^{f,i}$ never goes below $-D_i$.
This process repeats as necessary.

In general, the artificial spending is given by
\[
    \bar S_0^{f,i} \coloneqq 0, \quad \bar c_t^{f,i} \coloneqq \begin{cases}
        c^{f,i}_t & \bar S^{f,i}_{t-1} + (c^{f,i}_t - b_t^{i}) \geq - D_i \\
        b_t^{i} - \bar S^{f,i}_{t-1} - D_i & \text{o.w.}
    \end{cases}, \quad \bar S_t^{f,i} \coloneqq \bar S^{f,i}_{t-1} + (\bar c_t^{f,i} - b_t^{i})
\]
Equivalently,
\[
    \bar S_0^{f,i} \coloneqq 0, \quad \delta_t^{f,i} \coloneqq \left(- D_i-(\bar S^{f,i}_{t-1}+ c^{f,i}_t-b_t^{i})\right)_+, \quad \bar c_t^{f,i} \coloneqq c^{f,i}_t + \delta^{f,i}_t, \quad \bar S_t^{f,i} \coloneqq \bar S^{f,i}_{t-1} + (\bar c_t^{f,i} - b_t^{i})
\]

The main concern with artificially inflating costs is that the expert's artificial spending may exceed the baseline by more than $D_i$ at a later time.
But note that for the expert's spending to go from more than $D_i$ below the baseline's to more than $D_i$ above, the increase in the difference in spending must exceed $2D_i$.
This is exactly what Lemma~\ref{lem:subpace-char} precludes.

\begin{lemma}\label{lemma:raising}
Suppose each expert's spending pattern in resource $i$ is only $D_i$-subpacing. Define 
\[
    \bar S_0^{f,i} \coloneqq 0, \quad \delta_t^{f,i} \coloneqq (- D_i-(\bar S^{f,i}_{t-1}+ c^{f,i}_t-b_t^{i}))_{+}, \quad \bar c_t^{f,i} \coloneqq c^{f,i}_t + \delta^{f,i}_t, \quad \bar S_t^{f,i} \coloneqq \bar S^{f,i}_{t-1} + (\bar c_t^{f,i} - b_t^{i}).
\]
Then, for every expert $f$, resource $i$ and round $t$,
\[
0 \leq c_t^{f,i}\le \bar c_t^{f,i}\le1,
\qquad
|\bar S_t^{f,i}|\le D_i,
\qquad
r_t^f\le \alpha_i\bar c_t^{f,i}+\varepsilon_t^i.
\]
That is, replacing $c_t^{f,i}$ by $\bar c_t^{f,i}$ produces a paced instance with the same baseline spending pattern $b_i$, budget $B_i$, Kolmogorov bound $D_i$, reward-to-cost ratio $\alpha_i$, residuals $(\varepsilon_t^i)_t$, and violation budgets $V_i$.
\end{lemma}

\begin{proof}
Fix an expert-resource pair $(f,i)$ and suppress these indices.  Write
$A_t:=\sum_{s\le t}c_s$, $B_t:=\sum_{s\le t}b_s$, and $K_t:=\sum_{s\le t}\delta_s$.  
We show that 
\[
    \bar S_t = A_t - B_t + K_t \quad \text{and} \quad K_t = \left(\max_{0\le s\le t}(B_s - A_s) - D_i\right)_+
\]
by induction.
If these equalities hold in round $t-1$, then the update gives
\begin{align*}
K_t
&=K_{t-1}+(- D_i-(\bar S^{f,i}_{t-1}+ c^{f,i}_t-b_t^{i}))_{+} \tag{definition of $K_t$ and $\delta_t$} \\
&=K_{t-1}+(B_t - A_t - K_{t-1} - D_i)_+ \tag{IH} \\
&=\max\{K_{t-1},B_t-A_t-D_i\} \\
&=\left(\max_{0\le s\le t}(B_s - A_s) - D_i\right)_+
\quad
\end{align*}
Meanwhile,
\begin{align*}
    \bar S_t 
        &= \bar S_{t-1} + (\bar c_t - b_t) \tag{definition of $\bar S_t$} \\
        &= A_{t-1} - B_{t-1} + K_{t-1}  + (c_t + \delta_t - b_t) \tag{IH; definition of $\bar c_t$} \\
        &= A_t - B_t + K_t \tag{definitions of $A_t$, $B_t$, and $K_t$}
\end{align*}

We now prove that the expert's artificial spending pattern is $D_i$-paced, i.e. $|\bar S_t|\le D_i$ for every prefix.  Since the original spending pattern $S_t = A_t - B_t$ is $D_i$-subpacing, \Cref{lem:subpace-char} gives
$A_t-B_t\le D_i$ and
$(A_t-A_s)-(B_t-B_s)\le2D_i$ for all $0\le s<t$.
For the lower bound, $K_t\ge B_t-A_t-D_i$, hence, $\bar S_t=A_t-B_t+K_t\ge-D_i$.
For the upper bound, if $\max_{s\le t}(B_s-A_s)\le D_i$, then $K_t=0$ and $\bar S_t=A_t-B_t\le D_i$.  Otherwise choose $s\le t$ attaining $\max_{r\le t}(B_r-A_r)>D_i$.  Then $K_t=B_s-A_s-D_i$.  If $s=t$, this gives $\bar S_t=-D_i$.  If $s<t$, then
\[
\bar S_t
=(A_t-A_s)-(B_t-B_s)-D_i
\le2D_i-D_i=D_i.
\]
Thus, $|\bar S_t|\le D_i$ for every $t$.

By construction, $\delta_t\ge0$, so $\bar c_t=c_t+\delta_t\ge c_t \geq 0$.  Also, $\delta_t=K_t-K_{t-1}$.  If $K_t=K_{t-1}$, then $\bar c_t=c_t\le1$.  If $K_t>K_{t-1}$, then $K_t=B_t-A_t-D_i$ while $K_{t-1}\ge B_{t-1}-A_{t-1}-D_i$, and therefore,
\[
0<\delta_t=K_t-K_{t-1}\le b_t-c_t.
\]
Hence $\bar c_t=c_t+\delta_t\le b_t\le1$.  Finally, monotonicity of costs implies that $r_t^f\le\alpha_i c_t^{f,i}+\varepsilon_t^i\le\alpha_i\bar c_t^{f,i}+\varepsilon_t^i$.
\end{proof}

\begin{theorem}
\label{thm:raising}
Suppose each expert's spending pattern in resource $i$ is only $D_i$-subpacing. Define 
\[
    \bar S_0^{f,i} \coloneqq 0, \quad \delta_t^{f,i} \coloneqq ((b_t^{i} - c^{f,i}_t) - \bar S^{f,i}_{t-1} - D_i)_+, \quad \bar c_t^{f,i} \coloneqq c^{f,i}_t + \delta^{f,i}_t, \quad \bar S_t^{f,i} \coloneqq \bar S^{f,i}_{t-1} + (\bar c_t^{f,i} - b_t^{i}).
\]
Algorithms~\ref{algorithm:core-ORA} and~\ref{algorithm:core-OLRC} with $c_t^{f,i}$ replaced by $\bar c_t^{f,i}$ everywhere\footnote{Including their hard stop conditions.} and the same choice of parameters as in Theorem~\ref{thm:core} are budget feasible via fractional play and satisfy
\begin{align*}
\Reg_T^f&\le \mathsf{Tail}
+O\left(D_{\eta,\infty}\sqrt{\log F\,\log(m+1)}\right)
&&\text{with pre-decision feedback}\\
\Reg_T^f&\le \mathsf{Tail}
+O\left((1+\eta_{\max})\sqrt{T\log F}
+D_{\eta,\infty}\sqrt{\log F\,\log(m+1)}\right)
&&\text{without pre-decision feedback}
\end{align*}
for every expert $f$ simultaneously, where
\[
\mathsf{Tail}
:=\max_i\left\{\eta_i(1+D_i)+\indicator{\alpha_i<\beta_i}V_i\right\}.
\]
Moreover, if the algorithms are subpacing with respect to the artificial costs and some Kolmogorov bound, then they are also subpacing with respect to the original cost and the same Kolmogorov bound.
\end{theorem}

\begin{proof}
Fractional budget feasibility and the regret guarantees follow from Lemma~\ref{lemma:raising}: the latter are immediate from the reduction implied by Lemma~\ref{lemma:raising}, while the hard stop condition with respect to the artificial costs implies that $\sum_t \bar c_t^i(x_t) \leq B_i$, so point-wise domination of original costs by artificial costs (from Lemma~\ref{lemma:raising}) implies that
\[
    \sum_t c_t^i(x_t)\le \sum_t\bar c_t^i(x_t) \le B_i
\]
Point-wise domination and \Cref{lem:subpace-char} then imply the final part of the result.
\end{proof}

\subsection{Extension 2: Subpacing Algorithm}
\label{subsec:extension-subpacing-algo}

The spending of the core algorithms may not be $D_i$-subpacing with respect to the baseline spending pattern because the potential sums the regret term and an overspending penalty term: if the core algorithms accumulate large negative regret with respect to an expert mixture $w$, then they may exceed the spending of that mixture by a large margin.

To prevent this from happening, we force the core algorithms to play the null action if their spending exceeds the baseline by too much.
The resulting loss in reward from not following the controller's recommendation is then bounded using the reward-to-cost ratio $\alpha_i$. This is why we choose $\lambda_i = \alpha_i$ instead of $\eta_i$, since $\beta_i$ offers no control over the rewards missed because of overspending too much compared to the baseline, which can occur in rounds before the hard-stop. 

Additionally, for the algorithms to be subpacing, we also need to prevent the algorithms from underspending the baseline by a lot, followed by spending a lot in a short interval, which violates subpacing but does not show up as cumulative overspending. This is achieved by adjusting the baseline so it lowers if the algorithms underspend by too much. 

The subpacing algorithm for the ORA setting is shown in \Cref{alg:sp-pre} and the algorithm for the OLRC setting is shown in \Cref{alg:sp-no-pre}.

Write
\[
D_{\alpha,\infty}:=\max_{i\in[m]}\alpha_iD_i,
\qquad
\alpha_{\max}:=\max_{i\in[m]}\alpha_i.
\]

The core variables $Y_t^i$ and $E_t^i$ from \Cref{sec:near-optimal-algo} retain their original definitions. For this extension, we introduce a clipped cumulative deviation $Y_t^{i,\mathrm{sp}}$ defined as 
\[Y_0^{i,\mathrm{sp}}:=0\qquad Y_t^{i,\mathrm{sp}}=\max\{-W_i,\,Y_{t-1}^{i,\mathrm{sp}}+c_t^i(p_t)-b_t^i\}\]
for some $W_i>0$ chosen later and define
\[
E_t^{i,\mathrm{sp}}(w)
:=Y_t^{i,\mathrm{sp}}-\sum_f w_fS_t^{f,i}.
\]
Analogously, define
\begin{equation}
\Phi_t^{\mathrm{sp}}(w,v)
:=\sum_f w_fR_t^f
 +\sum_{i=1}^m\alpha_i v_iE_t^{i,\mathrm{sp}}(w)
 -\rho_x\Psi_x(w)-\rho_v\Psi_v(v),
\qquad
M_t^{\mathrm{sp}}:=\max_{w,v}\Phi_t^{\mathrm{sp}}(w,v),
\label{eq:sp-potential}
\end{equation}
where, as in \Cref{sec:near-optimal-algo}, all maxima over $(w,v)$ are over $\Delta(\mathcal F)\times\Delta_{m+1}$. For the pre-decision controller, define
\begin{equation}
F_t^{\mathrm{raw}}(x,w,v)
:=\Phi_{t-1}^{\mathrm{sp}}(w,v)
+r_t(w)-r_t(x)
+\sum_{i=1}^m\alpha_iv_i\bigl(c_t^i(x)-c_t^i(w)\bigr).
\label{eq:sp-one-step}
\end{equation}
For common parameters $\rho_x,\rho_v>0$, set
\begin{equation}
W_i:=1+D_i+\frac{2+\rho_v\log(4mT)}{\alpha_i},
\qquad i\in[m].
\label{eq:unified-W}
\end{equation}
We call round $t$ \emph{safe} if $Y_{t-1}^{i,\mathrm{sp}}\le W_i-1$ for every $i\in[m]$, and \emph{unsafe} otherwise. Before the hard stop, the learner plays the controller's recommendation $p_t=\pi_t(x_t)$ on a safe round and the null action $p_t=\perp$ on an unsafe round. Once the hard stop takes effect, the learner plays $p_t=\perp$ forever. $Y_t^{i,\rm sp}$ is always updated using the actual fractional cost $c_t^i(p_t)$.

\begin{algorithm}[ht]
\caption{Subpacing algorithm extension with pre-decision feedback.}
\label{alg:sp-pre}
Initialize $R_0^f=0$, $S_0^{f,i}=0$, and $Y_0^{i,\mathrm{sp}}=0$ for every $f,i$\;
\For{$t=1,\ldots,T$}{
    Receive the expert advice $q_t^f$ and observe $(r_t,c_t^1,\ldots,c_t^m)$\;
    \eIf{the hard stop is active, or $C_{t-1}^i+1>B_i$ for some $i$}{
        Activate the hard stop and set $p_t=\perp$\;
    }{
        \eIf{$Y_{t-1}^{i,\mathrm{sp}}\le W_i-1$ for every $i\in[m]$}{
            Choose $x_t\in\arg\min_{x\in\Delta(\mathcal F)}\max_{w,v}F_t^{\mathrm{raw}}(x,w,v)$\;
            Set $p_t=\pi_t(x_t)$\;
        }{
            Set $p_t=\perp$\;
        }
    }
    Update, for every $f,i$,
    \[
    \begin{aligned}
        R_t^f&=R_{t-1}^f+r_t^f-r_t(p_t),
        &S_t^{f,i}&=S_{t-1}^{f,i}+c_t^{f,i}-b_t^i,\\
        Y_t^{i,\mathrm{sp}}&=\max\{-W_i,\,Y_{t-1}^{i,\mathrm{sp}}+c_t^i(p_t)-b_t^i\}.
    \end{aligned}
    \]
}
\end{algorithm}

\begin{algorithm}[ht]
\caption{Subpacing algorithm extension without pre-decision feedback.}
\label{alg:sp-no-pre}
Initialize $R_0^f=0$, $S_0^{f,i}=0$, and $Y_0^{i,\mathrm{sp}}=0$ for every $f,i$\;
\For{$t=1,\ldots,T$}{
    Receive the expert advice $q_t^f$\;
    \eIf{the hard stop is active, or $C_{t-1}^i+1>B_i$ for some $i$}{
        Activate the hard stop and set $p_t=\perp$\;
    }{
        Choose $(x_t,v_t)\in\arg\max_{w,v}\Phi_{t-1}^{\mathrm{sp}}(w,v)$\;
        \eIf{$Y_{t-1}^{i,\mathrm{sp}}\le W_i-1$ for every $i\in[m]$}{
            Set $p_t=\pi_t(x_t)$\;
        }{
            Set $p_t=\perp$\;
        }
    }
    Observe $(r_t,c_t^1,\ldots,c_t^m)$\;
    Update, for every $f,i$,
    \[
    \begin{aligned}
        R_t^f&=R_{t-1}^f+r_t^f-r_t(p_t),
        &S_t^{f,i}&=S_{t-1}^{f,i}+c_t^{f,i}-b_t^i,\\
        Y_t^{i,\mathrm{sp}}&=\max\{-W_i,\,Y_{t-1}^{i,\mathrm{sp}}+c_t^i(p_t)-b_t^i\}.
    \end{aligned}
    \]
}
\end{algorithm}

\begin{lemma}
\label{lem:kl-variational}
Let $p\in\Delta_d$ have full support, let $s\in\mathbb R^d$, and let $z>0$. Then
\begin{equation}
\sup_{w\in\Delta_d}
\left\{\langle w,s\rangle-zD_{\mathrm{KL}}(w\|p)\right\}
=z\log\left(\sum_{k=1}^d p_k e^{s_k/z}\right).
\label{eq:kl-variational}
\end{equation}
The unique maximizer is
\begin{equation}
w_k^*=\frac{p_k e^{s_k/z}}{\sum_{j=1}^d p_j e^{s_j/z}}.
\label{eq:kl-softmax}
\end{equation}
\end{lemma}

\begin{proof}
Let $Z:=\sum_{k=1}^d p_k e^{s_k/z}$ and
$p_k^*:=p_k e^{s_k/z}/{Z}$.
For every $w\in\Delta_d$,
\begin{align*}
D_{\mathrm{KL}}(w\|p^*)-D_{\mathrm{KL}}(w\|p)
= \sum_{k=1}^{d} w\log \frac{p_k}{p^*_k}
= -\frac{1}{z}\langle w,s\rangle + \log Z.
\end{align*}
Therefore
\[
\langle w,s\rangle-zD_{\mathrm{KL}}(w\|p)
=z\log Z-zD_{\mathrm{KL}}(w\|p^*).
\]
The result follows from nonnegativity of KL divergence, with equality if and only if $w=p^*$.
\end{proof}

\begin{theorem}
\label{thm:subpacing-algo}
Assume the experts are $D_i$-paced in every resource $i$. Let $\rho_x,\rho_v>0$ satisfy
\begin{equation}
\rho_x\rho_v\ge 4D_{\alpha,\infty}^2
\label{eq:sp-curvature-condition}
\end{equation}
and 
\begin{equation}
W_i=1+D_i+\frac{2+\rho_v\log(4mT)}{\alpha_i},
\qquad i\in[m].
\end{equation}
With pre-decision feedback, run Algorithm~\ref{alg:sp-pre}; without pre-decision feedback, run Algorithm~\ref{alg:sp-no-pre}.

Define $V_{\mathrm{mix}}=\sum_{i=1}^m V_i$. In either regime, the actual fractional cost path satisfies every hard budget and is $W_i$-subpacing in every resource $i$. Moreover, simultaneously for every expert $f$, Algorithm~\ref{alg:sp-pre} satisfies
\begin{equation}
\operatorname{Reg}_T^f
\le
\rho_x\log F+\rho_v\log(m+1)
+\max_i\eta_i(1+D_i)
+V_{\mathrm{mix}}
+\frac{1+\alpha_{\max}}{4},
\label{eq:sp-pre-regret}
\end{equation}
and Algorithm~\ref{alg:sp-no-pre} satisfies
\begin{equation}
\operatorname{Reg}_T^f
\le
\rho_x\log F+\rho_v\log(m+1)
+\frac{T(1+\alpha_{\max})^2}{\rho_x}
+\max_i\eta_i(1+D_i)
+V_{\mathrm{mix}}
+\frac{1+\alpha_{\max}}{4}.
\label{eq:sp-no-pre-regret}
\end{equation}
Optimizing $\rho_x,\rho_v$ for regret bounds gives
\[
\operatorname{Reg}_T^f
=\widetilde O_{m}\!\left(D_{\alpha,\infty}\sqrt{\log F}+\alpha_{\max}+V_{\mathrm{mix}}\right)
\]
and 
\[
W_i
=
\widetilde O_{m,T}\!\left(
1+ D_i
+
\frac{1}{\alpha_i}
+
\frac{D_{\alpha,\infty}\sqrt{\log F}}{\alpha_i}
\right).
\]

with pre-decision feedback, and
\[
\operatorname{Reg}_T^f
=\widetilde O_{m}\!\left(\bigl((1+\alpha_{\max})\sqrt T+D_{\alpha,\infty}\bigr)\sqrt{\log F}+V_{\mathrm{mix}}\right)
\]
and
\[
W_i
=
\widetilde O_{m,T}\!\left(1+
D_i
+
\frac{1}{\alpha_i}
+
\frac{D_{\alpha,\infty}^2\sqrt{\log F}}
{\alpha_i\bigl(D_{\alpha,\infty}+(1+\alpha_{\max})\sqrt{T}\bigr)}
\right).
\]
without pre-decision feedback. The notations $\widetilde O_{m}(\cdot)$ and $\tilde{O}_{T, m}(\cdot)$ suppress polylogarithmic factors in their subscripts.
\end{theorem}

\begin{proof}
We first prove a subpacing witness, then bound the growth of the potential, and finally convert the prefix potential bound into full-horizon regret.

\paragraph{Subpacing witness.}
We claim that
\begin{equation}
-W_i\le Y_t^{i,\mathrm{sp}}\le W_i
\qquad\text{for every }i,t.
\label{eq:Y-sp-invariant}
\end{equation}
The lower bound follows directly from the definition. For the upper bound, we prove by induction. If round $t$ is safe, then $Y_{t-1}^{i,\mathrm{sp}}\le W_i-1$, and therefore
\[
Y_{t-1}^{i,\mathrm{sp}}+c_t^i(p_t)-b_t^i
\le Y_{t-1}^{i,\mathrm{sp}}+1
\le W_i.
\]
If round $t$ is unsafe, then $c_t^i(p_t)=0$, so
\[
Y_{t-1}^{i,\mathrm{sp}}+c_t^i(p_t)-b_t^i
=Y_{t-1}^{i,\mathrm{sp}}-b_t^i
\le Y_{t-1}^{i,\mathrm{sp}}
\le W_i,
\]
where the last inequality is the induction hypothesis. Thus \eqref{eq:Y-sp-invariant} holds.

For each $i,t$, define
\[
k_t^i:=\Bigl(-W_i-\bigl(Y_{t-1}^{i,\mathrm{sp}}+c_t^i(p_t)-b_t^i\bigr)\Bigr)_+,
\qquad
d_t^i:=b_t^i-k_t^i.
\]
Intuitively, $k_t^i$ is how much is clipped by the lower bound of $-W_i$ during round $t$, and we adjust the baseline downward by this clipped amount. We claim that for every resource $i$, $(d_t^i)_{t}$ is a subpacing witness of the algorithm, i.e. the algorithm is $W_i$-pacing with respect to $d^i$.

Since $Y_{t-1}^{i,\mathrm{sp}}+c_t^i(p_t)-b_t^i\ge -W_i-b_t^i$,  we have $0\le k_t^i\le b_t^i$ and therefore $d_t^i\in[0,b_t^i]$. This definition of $Y_t^{i,\mathrm{sp}}$ implies
\[
Y_t^{i,\mathrm{sp}}
=Y_{t-1}^{i,\mathrm{sp}}+c_t^i(p_t)-d_t^i,
\]
and hence
\begin{equation}
Y_t^{i,\mathrm{sp}}
=\sum_{s\le t}\bigl(c_s^i(p_s)-d_s^i\bigr).
\label{eq:sp-witness}
\end{equation}
Together with \eqref{eq:Y-sp-invariant}, this proves directly that the actual fractional resource-$i$ cost path $(c_t^i(p_t))_t$ is $W_i$-subpacing. The hard stop gives fractional budget feasibility.

\paragraph{Bounding potential difference due to clipping.}
On a round before the hard stop, let $\Phi_t^{\mathrm{raw}}$ denote the potential obtained after updating $R_t^f$ and $S_t^{f,i}$ and replacing $Y_t^{i,\mathrm{sp}}$ by its unclipped value $Y_{t}^{i,\mathrm{raw}} := Y_{t-1}^{i,\mathrm{sp}}+c_t^i(p_t)-b_t^i$. More formally, define
\begin{align*}
    \Phi_t^{\mathrm{raw}}(w,v)
:=\sum_f w_fR_t^f
 +\sum_{i=1}^m\alpha_i v_i E_t^{i,\mathrm{raw}}(w)
 -\rho_x\Psi_x(w)-\rho_v\Psi_v(v),
\end{align*}
where \[E_t^{i,\mathrm{raw}}(w):=Y_t^{i,\mathrm{raw}}-\langle w, S_t^{i} \rangle.\]
Let \[M_t^{\mathrm{raw}}:=\max_{w,v}\Phi_t^{\mathrm{raw}}(w,v).\] 

We want to bound the potential increase caused by clipping in a round before the hard stop. Fix $w$ and suppose $Y_{t}^{i,\mathrm{sp}}$ was clipped from $Y_{t}^{i,\mathrm{raw}}\leq -W_i$ to $Y_t^{i,\mathrm{sp}}=-W_i$ at the end of round $t$. Applying Lemma~\ref{lem:kl-variational} to the $v$-maximization part of the potential function gives
\begin{align*}
\sup_{v\in\Delta_{m+1}}
\left\{\langle v,(0,\alpha\odot E_t^{\mathrm{sp}})\rangle-\rho_v D_{\mathrm{KL}}(v\|U_{m+1})\right\}
=\rho_v\log\left(\frac{1}{m+1}+ \sum_{j=1}^{m} \frac{1}{m+1} e^{\alpha_j E_t^{j,\mathrm{sp}}/\rho_v}\right).
\end{align*}
where $\alpha\odot E_t^{\mathrm{sp}}$ is the pointwise product of $(\alpha_i)_{i=1}^{m}$ and $(E_t^{i,\mathrm{sp}})_{i=1}^{m}$ and we add a 0 coordinate in front. 
To isolate the effect of each coordinate of  $\alpha\odot E_t^{\mathrm{sp}}$, we define
\[V(s)=\sup_{v\in\Delta_{m+1}}
\left\{\langle v,(0,s)\rangle-\rho_v D_{\mathrm{KL}}(v\|U_{m+1})\right\}
=\rho_v\log\left(\frac{1}{m+1}+ \sum_{j=1}^{m} \frac{1}{m+1} e^{s_j/\rho_v}\right)
\]
for $s\in \RR^m$.
Then for the clipped coordinate $i\in [m]$, we have
\begin{align*}
    \frac{\partial V(s)}{\partial s_i} = \frac{e^{s_i/\rho_v}}{1+\sum_{j=1}^{m}e^{s_j/\rho_v}}\leq e^{s_i/\rho_v}.
\end{align*}
Since $Y_{t}^{i,\mathrm{raw}}\leq -W_i$, for any $y\in [Y_{t}^{i,\mathrm{raw}}, -W_i]$, we have 
\[
\alpha_i\bigl(y-\langle w,S_t^i\rangle\bigr)
\le \alpha_i\bigl(-W_i-\langle w,S_t^i\rangle\bigr)
\le \alpha_i(-W_i+D_i)
\le -\rho_v\log(4mT).
\]
Thus 
\[ \frac{\partial V(s)}{\partial s_i}\leq \frac{1}{4mT}\] for $s_i=\alpha_i\bigl(y-\langle w,S_t^i\rangle\bigr)$ and $y\in [Y_{t}^{i,\mathrm{raw}}, -W_i]$.
This implies clipping the coordinate $i$ affects $V(s)$ by at most 
\begin{align*}
\frac{1}{4mT}\alpha_ik_t^i\leq  \frac{1}{4mT}\alpha_ib_t^i \leq  \frac{\alpha_i}{4mT}.
\end{align*}
We apply the same argument for each clipped coordinate one by one, and the resulting change in $V(s)$ is bounded by 
\begin{align*}
   V(\alpha\odot (Y_{t}^{\mathrm{sp}}-\langle w, S_t\rangle)-V(\alpha\odot (Y_{t}^{\mathrm{raw}}-\langle w, S_t\rangle)
   \le  \frac{\alpha_{\max}}{4T}.
\end{align*}
on each round.
Since this is true for any $w$ and 
\begin{align*}
&M_t^{\mathrm{sp}}
=\max_w \sum_f w_fR_t^f-\rho_x\Psi_x(w) +  V(\alpha\odot (Y_{t}^{\mathrm{sp}}-\langle w, S_t\rangle)\\
&M_t^{\mathrm{raw}}
=\max_w \sum_f w_fR_t^f-\rho_x\Psi_x(w) +  V(\alpha\odot (Y_{t}^{\mathrm{raw}}-\langle w, S_t\rangle),
\end{align*}
we conclude
\begin{equation}
M_t^{\rm sp}-M_t^{\rm raw}\leq \frac{\alpha_{\max}}{4T}
\label{eq:reflection-leakage}
\end{equation}

\paragraph{Potential growth without pre-decision feedback.}

Suppose first that round $t$ is safe and the hard stop is inactive. Then $p_t=\pi_t(x_t)$, so $r_t(p_t)=r_t(x_t)$ and $c_t^i(p_t)=c_t^i(x_t)$. Before clipping in round $t$, we have
\[
\Phi_t^{\mathrm{raw}}(w,v)-\Phi_{t-1}^{\mathrm{sp}}(w,v)
=\left\langle r_t-\sum_{i=1}^m\alpha_iv_ic_t^i,\,w-x_t\right\rangle
\le (1+\alpha_{\max})\|w-x_t\|_1.
\]

Since changing from $Y_t^i$ to $Y_t^{i,\mathrm{sp}}$ in the potential function only changes the term $\alpha_iv_iY_t^{i,\mathrm{sp}}$, which is affine in $v$, \Cref{lem:curvature-lambda} still holds for this modified potential function with $\lambda_i=\alpha_i$, provided that the regularization parameters satisfy \eqref{eq:sp-curvature-condition}.  In particular, since $(x_t,v_t)$ maximizes $\Phi_{t-1}^{\rm sp}(w,v)$, then for every feasible $(w,v)$,
\begin{equation}
\Phi_{t-1}^{\mathrm{sp}}(w,v)
\le
M_{t-1}^{\mathrm{sp}}
-\frac{\rho_x}{2}D_{\mathrm{KL}}(w\|x_t)
-\frac{\rho_v}{2}D_{\mathrm{KL}}(v\|v_t).
\label{eq:sp-curvature}
\end{equation}
Thus
\begin{align*}
\Phi_t^{\mathrm{raw}}(w,v)-M_{t-1}^{\mathrm{sp}}
&= \Phi_t^{\mathrm{raw}}(w,v)-\Phi_{t-1}^{\mathrm{sp}}(w,v)+\Phi_{t-1}^{\mathrm{sp}}(w,v)-M_{t-1}^{\mathrm{sp}}\\
&\leq (1+\alpha_{\max})\|w-x_t\|_1
-\frac{\rho_x}{4}\|w-x_t\|_1^2.
\end{align*}
where the last line follows from Pinsker's Inequality and non-negativity of KL divergence.

Optimizing over $\|w-x_t\|_1\ge0$ yields
\begin{equation}
M_t^{\mathrm{raw}}-M_{t-1}^{\mathrm{sp}}
\le
\frac{(1+\alpha_{\max})^2}{\rho_x}.
\label{eq:safe-no-pre}
\end{equation}

Now suppose round $t$ is unsafe. Choose $j$ with $Y_{t-1}^{j,\mathrm{sp}}>W_j-1$. Since $\langle x_t,S_{t-1}^j\rangle\le D_j$,
\begin{equation}
\alpha_jE_{t-1}^{j,\mathrm{sp}}(x_t)
>\alpha_j(W_j-1-D_j)
=2+\rho_v\log(4mT).
\label{eq:unsafe-score-no-pre}
\end{equation}
Because $(x_t,v_t)$ maximizes $\Phi_{t-1}^{\mathrm{sp}}$, fixing $x_t$ shows that $v_t$ maximizes the $v$-dependent terms. Lemma~\ref{lem:kl-variational} therefore gives
\begin{equation}
v_{t,0}
=\frac{1}{1+\sum_{i=1}^m\exp(\alpha_iE_{t-1}^{i,\mathrm{sp}}(x_t)/\rho_v)}
\le
\frac{e^{-2/\rho_v}}{4mT}.
\label{eq:dummy-small}
\end{equation}
On an unsafe round before the hard stop, $p_t=\perp$, so $r_t(p_t)=0$ and $c_t^i(p_t)=0$ for every $i$. In this case,
\begin{align}
\label{eq:sub-unsafe-potential-change}
\Phi_t^{\mathrm{raw}}(w,v)-\Phi_{t-1}^{\mathrm{sp}}(w,v)
=\left\langle r_t-\sum_{i=1}^m\alpha_iv_ic_t^i,\,w\right\rangle.    
\end{align}

Let $\epsilon_t^{\max}:=\max_i\epsilon_t^i$. For every $(w,v)$, \Cref{assumption:soft-conversion} implies $r_t(w)\le\alpha_ic_t^i(w)+\epsilon_t^{\max}$. Then multiplying this by $v_i$ and summing over $i\in[m]$ gives
\[
(1-v_0)r_t(w)
\le
\sum_{i=1}^m\alpha_iv_ic_t^i(w)+(1-v_0)\epsilon_t^{\max}.
\]
Since $r_t(w)\le1$, we have
\begin{align}
r_t(w)-\sum_{i=1}^m\alpha_iv_ic_t^i(w)
\le v_0+\epsilon_t^{\max}\\
\implies \Phi_t^{\mathrm{raw}}(w,v)-\Phi_{t-1}^{\mathrm{sp}}(w,v)\leq v_0+\epsilon_t^{\max}.
\label{eq:null-round-soft}
\end{align}

Combining \eqref{eq:sp-curvature} and \eqref{eq:null-round-soft} gives
\begin{align*}
\Phi_t^{\mathrm{raw}}(w,v)-M_{t-1}^{\mathrm{sp}}
&= \Phi_t^{\mathrm{raw}}(w,v)-\Phi_{t-1}^{\mathrm{sp}}(w,v)+\Phi_{t-1}^{\mathrm{sp}}(w,v)-M_{t-1}^{\mathrm{sp}}\\
&\leq v_0+\epsilon_t^{\max} -\frac{\rho_x}{2}D_{\mathrm{KL}}(w\|x_t)
-\frac{\rho_v}{2}D_{\mathrm{KL}}(v\|v_t)
\end{align*}

Optimizing over $(w,v)$ gives
\[
\begin{aligned}
M_t^{\mathrm{raw}}-M_{t-1}^{\mathrm{sp}}
\le \epsilon_t^{\max}
&+\sup_w\left\{-\frac{\rho_x}{2}D_{\mathrm{KL}}(w\|x_t)\right\}\\
&+\sup_v\left\{v_0-\frac{\rho_v}{2}D_{\mathrm{KL}}(v\|v_t)\right\}.
\end{aligned}
\]
The first supremum is zero. \Cref{lem:curvature-lambda} guarantees that $v_t$ has full support, so Lemma~\ref{lem:kl-variational} applies to the second supremum with $p=v_t$, $s=e_0$, and $z=\rho_v/2$. This gives
\begin{align*}
\sup_v\left\{v_0-\frac{\rho_v}{2}D_{\mathrm{KL}}(v\|v_t)\right\}
&=\frac{\rho_v}{2}\log\left(1+v_{t,0}(e^{2/\rho_v}-1)\right)\\
&\le \frac{\rho_v}{8mT}(1-e^{-2/\rho_v})\tag{by \eqref{eq:dummy-small}}\\
&\le \frac{1}{4mT}.\tag{$1-e^{-x}\le x$}
\end{align*}
Thus every unsafe round satisfies
\begin{equation}
M_t^{\mathrm{raw}}-M_{t-1}^{\mathrm{sp}}
\le
\epsilon_t^{\max}+\frac{1}{4mT}.
\label{eq:unsafe-no-pre}
\end{equation}
Since $M_0^{\mathrm{sp}}=0$, summing \eqref{eq:safe-no-pre}, \eqref{eq:unsafe-no-pre}, and \eqref{eq:reflection-leakage} through any prefix $\tau$ before the hard stop gives
\begin{equation}
M_\tau^{\mathrm{sp}}
\le
\frac{T(1+\alpha_{\max})^2}{\rho_x}
+\sum_{t\leq \tau}\epsilon_t^{\max}
+\frac{1+\alpha_{\max}}{4}.
\label{eq:no-pre-prefix}
\end{equation}

\paragraph{Potential growth with pre-decision feedback.}
Recall that $F_t^{\mathrm{raw}}(x,w,v)=\Phi_{t-1}^{\mathrm{sp}}(w,v)
+r_t(w)-r_t(x)
+\sum_{i=1}^m\alpha_iv_i\bigl(c_t^i(x)-c_t^i(w)\bigr)$.
On a safe round, $F_t^{\mathrm{raw}}$ is affine in $x$ and jointly concave in $(w,v)$ for fixed $x$ by \Cref{lem:curvature-lambda}, since changing from $Y_t^i$ to $Y_t^{i,\mathrm{sp}}$ in the potential function only changes the term $\alpha_iv_iY_t^{i,\mathrm{sp}}$, which is affine in $v$, and the regularization parameters satisfy \eqref{eq:sp-curvature-condition}. Sion's minimax theorem gives
\[
\begin{aligned}
M_t^{\mathrm{raw}}
&=\min_x\max_{w,v}F_t^{\mathrm{raw}}(x,w,v)
 =\max_{w,v}\min_xF_t^{\mathrm{raw}}(x,w,v)\\
&\le\max_{w,v}F_t^{\mathrm{raw}}(w,w,v)
 =M_{t-1}^{\mathrm{sp}}.
\end{aligned}
\]
where the first equality follows from the choice of $x_t$ in \Cref{alg:sp-pre}.
Thus a safe round causes no raw increase.

Now suppose round $t$ is unsafe and choose $j$ with $Y_{t-1}^{j,\mathrm{sp}}>W_j-1$. For every $w$,
\begin{equation}
\alpha_jE_{t-1}^{j,\mathrm{sp}}(w)
>\alpha_j(W_j-1-D_j)
=2+\rho_v\log(4mT).
\label{eq:unsafe-score-pre}
\end{equation}
Since the action chosen is the same in both regimes during an unsafe round, we can reuse the argument from the no-pre-decision feedback regime. Let $(w_t^\star,v_t^\star)$ maximize $\Phi_{t-1}^{\mathrm{sp}}$. In the no-pre-decision regime this is the pair chosen by the controller, while in the pre-decision regime it is used only for the analysis. Then we get 
\begin{align}
M_t^{\mathrm{raw}}-M_{t-1}^{\mathrm{sp}}
&\le \epsilon_t^{\max}
+\sup_w\left\{-\frac{\rho_x}{2}D_{\mathrm{KL}}(w\|w^*_t)\right\}+\sup_v\left\{v_0-\frac{\rho_v}{2}D_{\mathrm{KL}}(v\|v^*_t)\right\}\nonumber\\
&\leq \epsilon_t^{\max}+\frac{1}{4mT}.\label{eq:unsafe-pre}
\end{align}
following the same argument as the other regime.

Summing \eqref{eq:unsafe-pre} and \eqref{eq:reflection-leakage} through any prefix $\tau$ before the hard stop gives
\begin{equation}
M_\tau^{\mathrm{sp}}
\le
\sum_{t\leq \tau}\epsilon_t^{\max}
+\frac{1+\alpha_{\max}}{4}.
\label{eq:pre-prefix}
\end{equation}

\paragraph{From the prefix potential to regret.}
This part of the proof basically repeats the logic in \Cref{sec:complete-analysis} by noting that \Cref{lem:tail}  still holds if we have $E_\tau^{j,\mathrm{sp}}(e_f)\ge C_\tau^j-A_\tau^{f,j}$ (instead of having $E_\tau^{j}(e_f)= C_\tau^j-A_\tau^{f,j}$ as in the core algorithm) and setting $h_j=1$. For completion, we expand out the proofs.

If the hard stop never occurs, then by definition $\operatorname{Reg}_T^f=R_T^f$. Evaluating \eqref{eq:sp-potential} at $(e_f,e_0)$ gives
\[
\operatorname{Reg}_T^f
\le
M_T^{\mathrm{sp}}+\rho_x\log F+\rho_v\log(m+1).
\]

Now suppose the hard stop first takes effect on round $\tau+1$, and let resource $j$ trigger it. Thus $\tau$ is the last round before the hard stop, and by definition
\[
\operatorname{Reg}_\tau^f=R_\tau^f,
\qquad
C_\tau^j=\sum_{t\le\tau}c_t^j(p_t).
\]
Recall
\[
A_t^{f,j}:=\sum_{s\le t}c_s^{f,j},
\qquad
B_t^j:=\sum_{s\le t}b_s^j.
\]
By \eqref{eq:sp-witness},
\begin{align*}
E_\tau^{j,\mathrm{sp}}(e_f)
&=Y_\tau^{j,\mathrm{sp}}-S_\tau^{f,j}\\
&=C_\tau^j-\sum_{t\le\tau}d_t^j-(A_\tau^{f,j}-B_\tau^j)\\
&\ge C_\tau^j-A_\tau^{f,j}.\tag{$d_t^j\le b_t^j$}
\end{align*}

Since resource $j$ triggers the stop on round $\tau+1$, the stopping rule gives $B_j-C_\tau^j<1$, while pacing gives $A_T^{f,j}\le B_j+D_j$. Therefore
\begin{equation}
\begin{aligned}
\sum_{t>\tau}c_t^{f,j}
&=A_T^{f,j}-A_\tau^{f,j}\\
&\le (B_j-C_\tau^j)+(C_\tau^j-A_\tau^{f,j})+D_j\\
&<1+E_\tau^{j,\mathrm{sp}}(e_f)+D_j.
\end{aligned}
\label{eq:sp-suffix-cost}
\end{equation}
By \Cref{assumption:soft-conversion},
\begin{equation}
\sum_{t>\tau}r_t^f
\le
\alpha_j\bigl(E_\tau^{j,\mathrm{sp}}(e_f)+1+D_j\bigr)
+\sum_{t>\tau}\epsilon_t^j.
\label{eq:sp-suffix-alpha}
\end{equation}
If $\beta_j<\infty$, then
\[
\begin{aligned}
\sum_{t>\tau}b_t^j
&=B_j-B_\tau^j\\
&=(B_j-C_\tau^j)+(C_\tau^j-A_\tau^{f,j})+(A_\tau^{f,j}-B_\tau^j)\\
&<1+E_\tau^{j,\mathrm{sp}}(e_f)+D_j.
\end{aligned}
\]
Since $\sum_{t>\tau}r_t^f\le T-\tau\le\beta_j\sum_{t>\tau}b_t^j$, we also have
\begin{equation}
\sum_{t>\tau}r_t^f
\le
\beta_j\bigl(E_\tau^{j,\mathrm{sp}}(e_f)+1+D_j\bigr).
\label{eq:sp-suffix-beta}
\end{equation}
Combining \eqref{eq:sp-suffix-alpha} and \eqref{eq:sp-suffix-beta},
\begin{equation}
\sum_{t>\tau}r_t^f
\le
\eta_j\bigl(E_\tau^{j,\mathrm{sp}}(e_f)+1+D_j\bigr)
+\mathbf 1\{\alpha_j<\beta_j\}\sum_{t>\tau}\epsilon_t^j.
\label{eq:sp-suffix}
\end{equation}

Set $s:=\eta_j/\alpha_j\le1$ and evaluate the prefix potential at
\[
(w,v)=\bigl(e_f,(1-s)e_0+se_j\bigr).
\]
Using $\alpha_js=\eta_j$, $\Psi_x(e_f)=\log F$, and $\Psi_v((1-s)e_0+se_j)\le\log(m+1)$ gives
\begin{equation}
\operatorname{Reg}_\tau^f
\le
M_\tau^{\mathrm{sp}}+\rho_x\log F+\rho_v\log(m+1)
-\eta_jE_\tau^{j,\mathrm{sp}}(e_f).
\label{eq:sp-prefix-regret}
\end{equation}
After the hard stop the learner obtains zero reward, so
\[
\operatorname{Reg}_T^f
=\operatorname{Reg}_\tau^f+\sum_{t>\tau}r_t^f.
\]
Adding \eqref{eq:sp-suffix} and \eqref{eq:sp-prefix-regret} cancels the $E_\tau^{j,\mathrm{sp}}(e_f)$ terms. Moreover,
\[
\sum_{t\leq \tau}\epsilon_t^{\max}
+\mathbf 1\{\alpha_j<\beta_j\}\sum_{t>\tau}\epsilon_t^j
\le
\sum_{t=1}^T\epsilon_t^{\max}
\le V_{\mathrm{mix}}.
\]
Substituting \eqref{eq:pre-prefix} or \eqref{eq:no-pre-prefix} yields \eqref{eq:sp-pre-regret} and \eqref{eq:sp-no-pre-regret}.

For the optimized rates, take $\rho_v=4D_{\alpha,\infty}^2/\rho_x$. With pre-decision feedback, minimizing
\[
\rho_x\log F+\frac{4D_{\alpha,\infty}^2\log(m+1)}{\rho_x}
\]
gives the stated bound. Without pre-decision feedback, minimize
\[
\rho_x\log F
+\frac{4D_{\alpha,\infty}^2\log(m+1)+T(1+\alpha_{\max})^2}{\rho_x},
\]
and use $\sqrt{a+b}\le\sqrt a+\sqrt b$.
\end{proof}

\paragraph{Composing Extension 1 and 2.} We briefly explain why this extension composes with the online-raising reduction of Extension~1. Let $\bar c_t^{f,i}$ denote the raised expert costs, and define
\[
\bar S_t^{f,i}:=\sum_{s\le t}(\bar c_s^{f,i}-b_s^i),
\qquad
\bar c_t^i(w):=\sum_f w_f\bar c_t^{f,i}.
\]
Apply the same controller to the raised instance: replace every occurrence of $c_t^i(w)$ in the analytical controller by $\bar c_t^i(w)$ and every $S_t^{f,i}$ in the comparison state by $\bar S_t^{f,i}$. On a safe round before the hard stop, we update $Y_t^{i,\rm sp}$ using cost $\bar c_t^i(x_t)$; on an unsafe round or after the hard stop, it is updated using zero cost. The reward coordinate remains the true reward $r_t(p_t)$, since raising changes only costs.

For budget accounting, define the raised cumulative cost
\[
\bar C_{t-1}^i
:=\sum_{\substack{s<t:\ s\text{ is safe and}\\ s\text{ precedes the hard stop}}}
\bar c_s^i(x_s),
\]
and trigger the hard stop whenever $\bar C_{t-1}^i+1>B_i$ for some $i$. The physical distribution is still $p_t=\pi_t(x_t)$ on a safe round before the hard stop and $p_t=\perp$ otherwise. Hence, on every round,
\[
0\le c_t^i(p_t)
\le
\begin{cases}
\bar c_t^i(x_t),&\text{if $t$ is safe and precedes the hard stop},\\
0,&\text{otherwise}.
\end{cases}
\]
Since the raised experts are $D_i$-paced, lie in $[0,1]$, and satisfy the same soft reward-to-cost condition, the preceding potential and terminal arguments apply verbatim to the raised accounting path. They show that this raised path is $W_i$-subpacing and respects the raised hard budget. The true fractional cost path is pointwise dominated by the raised accounting path. Hence true budget feasibility follows from the raised hard budget, and \Cref{lem:subpace-char} transfers the $W_i$-subpacing guarantee to the true costs.

\subsection{Extension 3: sampled play with pathwise budgets and high-probability guarantees}
\label{sec:sampled-feasibility}

In this section, we show a straightforward extension to the sampling setting under full-information, where the algorithm gets realized rewards $r_t(a)$ from a draw $a\sim p_t$ and incurs realized cost $c_t^i(a)$. The algorithm hard stops when the realized costs exceed $B_i-1$ for some resource $i$. We use martingale concentration to show a high probability regret bound. If we additionally compose with Extension 2, then we also show a high probability subpacing bound.

\paragraph{Sampling the core fractional algorithm.}
Let $p_t\in\Delta(A)$ denote the round-$t$ distribution produced by the core fractional algorithm, where recall that hard stop was determined by the algorithm's fractional/expected spending. The sampled learner maintains an additional \emph{realized hard stop}. Let $\widehat c_t^i:=c_t^i(a_t)$ and 
\[
\widehat C_{t-1}^i:=\sum_{s<t}\widehat c_s^i
\]
be its realized cumulative cost in resource $i$. Before round $t$, if the realized hard stop is already active, or if $\widehat C_{t-1}^i+1>B_i$ for some $i$, the learner activates the realized hard stop and sets $a_t:=\perp$. Otherwise it samples $a_t\sim p_t$. The underlying fractional controller is still run as in \Cref{sec:near-optimal-algo}. The realized hard stop only truncates the sampled play and does not alter $p_t$ or the fractional controller's state.

Let $\chi_t\in\{0,1\}$ indicate that round $t$ has not entered the realized hard stop yet, so that $\chi_t=1$ exactly when $a_t$ is sampled from $p_t$. Define
\[
\widehat r_t:=r_t(a_t),\qquad
\widehat c_t^i:=c_t^i(a_t),\qquad
\mu_t^r:=\chi_t r_t(p_t),\qquad
\mu_t^i:=\chi_t c_t^i(p_t).
\]
Thus $\mu_t^r$ and $\mu_t^i$ are the expected reward and resource-$i$ cost of the sampled learner conditioned on all the information before its private round-$t$ draw (with both equal to zero when the realized hard stop is active).

For $V,L\ge 0$, define
\[
\operatorname{Fr}(V,L)
:=
\frac{L}{3}+\sqrt{2VL+\frac{L^2}{9}},
\]
and set
\[
Q_r:=\min\left\{T,\min_{i\in[m]}(\alpha_iB_i+V_i)\right\},
\qquad
Q_i:=\min\{B_i,T\},
\]
\[
\Gamma_r(\delta)
:=\operatorname{Fr}\!\left(Q_r,\log\frac{2}{\delta}\right),
\qquad
\Gamma_i(\delta)
:=\operatorname{Fr}\!\left(Q_i,\log\frac{4m}{\delta}\right).
\]
In particular,
\[
\Gamma_r(\delta)
=O\!\left(\sqrt{Q_r\log\frac{2}{\delta}}+\log\frac{2}{\delta}\right),
\qquad
\Gamma_i(\delta)
=O\!\left(\sqrt{Q_i\log\frac{4m}{\delta}}+\log\frac{4m}{\delta}\right).
\]
We use the following Freedman inequality \citep{freedman1975tail}: if $(Z_t)_{t\le T}$ is a martingale-difference sequence with $|Z_t|\le 1$ and
\[
\sum_{t=1}^T \mathbb E[Z_t^2\mid\mathcal H_t]\le V
\]
almost surely, then
\[
\Pr\!\left(\max_{s\le T}\sum_{t\le s}Z_t>\operatorname{Fr}(V,L)\right)
\le e^{-L}.
\]
The same statement applies to the lower tail by replacing $Z_t$ with $-Z_t$.

\begin{lemma}
\label{lem:sampling-deviations}
For every $\delta\in(0,1)$, with probability at least $1-\delta$, simultaneously for every prefix $s\le T$ and every resource $i\in[m]$,
\begin{equation}
\label{eq:concentration-event}
\sum_{t\le s}(\mu_t^r-\widehat r_t)\le \Gamma_r(\delta),
\qquad
\left|\sum_{t\le s}(\widehat c_t^i-\mu_t^i)\right|
\le \Gamma_i(\delta).
\end{equation}

\end{lemma}

\begin{proof}
Let $\mathcal H_t$ contain the history, the fractional proposal $p_t$, the indicator $\chi_t$, and the complete round-$t$ reward and cost vectors immediately before the private draw. In both feedback regimes, the adversary has fixed the round-$t$ reward and cost vectors before this draw. Hence
\[
\mathbb E[\widehat r_t\mid\mathcal H_t]=\mu_t^r,
\qquad
\mathbb E[\widehat c_t^i\mid\mathcal H_t]=\mu_t^i.
\]
Therefore
\[
Z_t^r:=\widehat r_t-\mu_t^r,
\qquad
Z_t^i:=\widehat c_t^i-\mu_t^i
\]
are martingale differences with absolute value at most one. Since rewards and costs lie in $[0,1]$, we have
\[
\mathbb E[(Z_t^r)^2\mid\mathcal H_t]=\operatorname{Var}(\widehat r_t\mid\mathcal H_t)\leq E[(\widehat r_t)^2\mid\mathcal H_t] \leq E[(\widehat r_t)\mid\mathcal H_t] = \mu_t^r
\]
and similarly, $\mathbb E[(Z_t^i)^2\mid\mathcal H_t]\le \mu_t^i$. The fractional hard stop gives $\sum_t c_t^i(p_t)\le B_i$, so
\[
\sum_t\mu_t^i\le \min\{B_i,T\}=Q_i.
\]
Moreover, by \Cref{assumption:soft-conversion}, for every resource $i$,
\[
\sum_t\mu_t^r
\le
\sum_t\chi_t\bigl(\alpha_i c_t^i(p_t)+\epsilon_t^i\bigr)
\le \alpha_iB_i+V_i,
\]
and trivially $\sum_t\mu_t^r\le T$. Thus the reward predictable variation is at most $Q_r$.

Apply Freedman's Inequality to $-Z_t^r$ with $L=\log(2/\delta)$ and to each of $Z_t^i$ and $-Z_t^i$ with $L=\log(4m/\delta)$. A union bound over the resulting $1+2m$ events proves the claim.
\end{proof}

\begin{lemma}
\label{lem:sampling-terminal-core}
Condition on \eqref{eq:concentration-event}. Let $\tau$ be the last round before either the fractional hard stop or the realized hard stop first takes effect. If $\tau<T$, then some resource $j$ satisfies
\begin{equation}
\label{eq:sampled-tail-budget}
B_j-C_\tau^j<1+\Gamma_j(\delta).
\end{equation}
Consequently, for every expert $f$,
\begin{equation}
\label{eq:sampled-tail-reward}
\sum_{t>\tau}r_t^f
\le
\eta_j\bigl(E_\tau^j(e_f)+1+D_j+\Gamma_j(\delta)\bigr)
+
\mathbf 1\{\alpha_j<\beta_j\}\sum_{t>\tau}\epsilon_t^j.
\end{equation}
\end{lemma}

\begin{proof}
Recall $C_s^i=\sum_{t\le s}c_t^i(p_t)$.
If the fractional hard stop is first, then its stopping rule gives $C_\tau^j+1>B_j$ for some resource $j$. Then \eqref{eq:sampled-tail-reward} holds even without the concentration term $\Gamma_j(\delta)$ by applying \Cref{lem:tail} with $h_j=1$.

If the realized hard stop is first, then for some resource $j$, $\sum_{t\le\tau}\widehat c_t^j+1>B_j$.
Both stops are inactive through round $\tau$, so $\chi_t=1$ for all $t\le\tau$ and therefore $\mu_t^j=c_t^j(p_t)$ on this prefix. Thus we have
\[
C_\tau^j
=
\sum_{t\le\tau}\mu_t^j
\ge
\sum_{t\le\tau}\widehat c_t^j-\Gamma_j(\delta)
>
B_j-1-\Gamma_j(\delta),
\]
which proves \eqref{eq:sampled-tail-budget}. Applying \Cref{lem:tail} with $h_j=1+\Gamma_j(\delta)$ gives \eqref{eq:sampled-tail-reward}.
\end{proof}

\begin{theorem}
\label{thm:sampling-core}
Let $\operatorname{Reg}_{\mathrm{frac}}$ denote the deterministic fractional regret bound from \Cref{thm:core} for the chosen feedback regime. The sampled learner satisfies every hard budget on every realization. Moreover, with probability at least $1-\delta$, simultaneously for every expert $f$,
\begin{equation}
\label{eq:sampled-regret}
\widehat{\operatorname{Reg}}_T^f
:=
\sum_{t=1}^T r_t^f-\sum_{t=1}^T\widehat r_t
\le
\operatorname{Reg}_{\mathrm{frac}}
+
\Gamma_r(\delta)
+
\max_{i\in[m]}\eta_i\Gamma_i(\delta).
\end{equation}
\end{theorem}

\begin{proof}
The realized hard stop guarantees pathwise budget feasibility. Condition on \eqref{eq:concentration-event}, and let $\tau$ be the last round before either stop first takes effect, with $\tau=T$ if neither stop occurs. Since both stops are inactive through $\tau$, we have $\mu_t^r=r_t(p_t)$ for every $t\le\tau$. Hence
\begin{equation}
\label{eq:sampled-tau-regret}    
\sum_{t\le\tau}(r_t^f-\widehat r_t) = \sum_{t\le\tau}(r_t^f-r_t(p_t))+\sum_{t\le\tau}(r_t(p_t)-\widehat r_t)
\le
R_\tau^f+\Gamma_r(\delta).
\end{equation}
If $\tau=T$, combining \eqref{eq:sampled-tau-regret} with \Cref{thm:core} proves \eqref{eq:sampled-regret}.

Now suppose $\tau<T$, and let $j$ be the resource from Lemma~\ref{lem:sampling-terminal-core}. The prefix-potential argument in the proof of \Cref{thm:general-lambda} is unchanged through round $\tau$. In particular, evaluating the prefix potential at $\bigl(e_f,(1-s_j)e_0+s_je_j\bigr)$ with $s_j=\eta_j/\lambda_j$ gives
\begin{equation}
\label{eq:sampled-tau-potential}
R_\tau^f+\eta_jE_\tau^j(e_f)
\le
\operatorname{Stab}_\tau+\rho_x\log F+\rho_v\log(m+1),
\end{equation}
where $\operatorname{Stab}_\tau=0$ with pre-decision feedback and
\[
\operatorname{Stab}_\tau
\le
\frac{T(1+\eta_{\max})^2}{\rho_x}
\]
without pre-decision feedback. Combining \eqref{eq:sampled-tail-reward},\eqref{eq:sampled-tau-regret} and \eqref{eq:sampled-tau-potential} yields
\[
\widehat{\operatorname{Reg}}_T^f
\le
\operatorname{Reg}_{\mathrm{frac}}
+
\Gamma_r(\delta)
+
\eta_j\Gamma_j(\delta),
\]
Since $\eta_j\Gamma_j(\delta)\le\max_i\eta_i\Gamma_i(\delta)$, \eqref{eq:sampled-regret} follows.
\end{proof}

We now state a corollary of composing three extensions.
\begin{corollary}
\label{thm:sampling-composable}
Let $\operatorname{Reg}_{\mathrm{frac}}$ denote the deterministic fractional regret bound of the chosen base algorithm: \Cref{thm:core} for the core algorithm, or \Cref{thm:subpacing-algo} when Extension~2 is used. Applying Extension~1 does not change this deterministic bound. Then the sampled learner satisfies every hard budget on every realization and, with probability at least $1-\delta$, simultaneously for every expert $f$,
\[
\widehat{\operatorname{Reg}}_T^f
\le
\operatorname{Reg}_{\mathrm{frac}}
+
\Gamma_r(\delta)
+
\max_{i\in[m]}\eta_i\Gamma_i(\delta).
\]
If Extension~2 is used, then on the same event the realized resource-$i$ cost path is $(W_i+\Gamma_i(\delta))$-subpacing for every $i\in[m]$. 
\end{corollary}

\begin{proof}[Proof sketch]
Extension~1 raises each $D_i$-subpacing expert to a virtual $D_i$-paced expert without changing any parameters entering the regret bound. Run Extension~2 on this raised instance. The algorithm is $W_i$-subpacing and budget-feasible, and since the true costs are pointwise dominated by the raised costs, downward closure transfers both properties to the true fractional path. Finally, Extension~3 samples this fractional path with a realized hard stop. The realized hard stop guarantees budget feasibility pathwise. Conditioned on \eqref{eq:concentration-event}, sampling increases the fractional regret bound by at most $
\Gamma_r(\delta)+\max_{i\in[m]}\eta_i\Gamma_i(\delta)$ by the same argument as in the proof of \Cref{thm:sampling-core} and observing that the claims of \Cref{lem:tail} still hold with Extension 2 (See \Cref{subsec:extension-subpacing-algo}). Sampling also preserves learner subpacing up to the concentration radius $\Gamma_i(\delta)$:  since $0\le\mu_t^i\le c_t^i(p_t)$ pointwise and $(c_t^i(p_t))_{t\le T}$ is $W_i$-subpacing, \Cref{thm:subpacing-algo} implies that $(\mu_t^i)_{t\le T}$ is $W_i$-subpacing. Let $d_t^i\in[0,b_t^i]$ be a corresponding witness. On the event \eqref{eq:concentration-event}, for every prefix $s$,
\[
\left|\sum_{t\le s}(\widehat c_t^i-d_t^i)\right|
\le
\left|\sum_{t\le s}(\mu_t^i-d_t^i)\right|
+
\left|\sum_{t\le s}(\widehat c_t^i-\mu_t^i)\right|
\le
W_i+\Gamma_i(\delta).
\]
Thus the realized sampled resource-$i$ path is $(W_i+\Gamma_i(\delta))$-subpacing.
\end{proof}

\section{Lower Bounds}
\label{sec:lower-bound}

In this section, we show that the regret of any algorithm against a set of experts must depend on the Kolmogorov distance between their spending patterns, regardless of whether the algorithm has pre-decision feedback.
Theorem~\ref{theorem:lower-bound-any-spending-pattern} states that for any spending patterns that are $D$ apart and for any algorithm, there exists an instance involving only two actions and two experts on which the algorithm incurs $\Omega(D)$ regret.
Theorem~\ref{theorem:lower-bound-number-experts-dependence} states that for any $D$ and for any number of experts $F$, there exists a single instance on which any algorithm incurs $\Omega(D\sqrt{\log F})$ regret.
Both results hold even in the ORA regime.
The difference between the two results, outside of the number of experts, is that the first result holds for any choice of spending patterns, while the construction in the second result chooses a specific baseline spending pattern.
The constructions for both results can be found in \cite{braverman2025newbenchmark}, but we outline them here to be self-contained.

\begin{theorem}[\cite{braverman2025newbenchmark}]\label{theorem:lower-bound-any-spending-pattern}
For any two spending patterns $\bm{s}, \bm{s}' \in [0,1]^T$ such that $\sum_t s_t = \sum_t s'_t = B$ and for any algorithm that spends at most $B$ in expectation, there exists a set of experts $\mathcal{F}$ and a sequence of rewards and costs such that the algorithm incurs expected regret at least $\kol(\bm{s},\bm{s}')/4$.
This lower bound holds with $\abs{\mathcal{F}} = 2$ and even if the algorithm can observe the rewards and costs of the current round before making its decision.
\end{theorem}

The main idea behind Theorem~\ref{theorem:lower-bound-any-spending-pattern} is the spend-or-save dilemma: if there exists a time $\tau$ when the cumulative spending of two spending patterns differ by $D$, then the expert that spent more before $\tau$ does better in a world where the most valuable rewards come before $\tau$, while the expert that spent less before $\tau$ does better in a world where the most valuable reward come after $\tau$.
If the rewards and costs up to time $\tau$ are identical between the two worlds (so an algorithm cannot distinguish between the two worlds), then the algorithm once again faces a spend-or-save dilemma proportional to the difference in spending between the two experts at time $\tau$.

\begin{proof}
Let $D = \kol(\bm{s}, \bm{s}')$ and $S_\tau = \sum_{t=1}^\tau s_t$.
Define $S'_\tau$ similarly for $\bm{s}'$.
The construction consists of two actions $\mathcal{A} = \{\top, \perp\}$ and two experts $\mathcal{F} = \{\bm{x}, \bm{x}'\}$ where
\begin{align*}
    x_t(\top) = s_t \quad&\mathrm{and}\quad x_t(\perp) = 1-s_t \\
    x'_t(\top) = s'_t \quad&\mathrm{and}\quad x'_t(\perp) = 1-s'_t.
\end{align*}
Let $\tau$ denote a round where $\abs{S_\tau - S'_\tau} = D$, and assume without loss of generality that $S_\tau - S'_\tau = D$.
Note that $\tau \leq T - D$ since 
\[
    \textstyle D = \sum_{t=1}^\tau (s_t - s'_t) \leq B - \sum_{t=1}^\tau s'_t = \sum_{t=\tau+1}^T s'_t \leq T - \tau.
\]
Define two sequences of rewards $\bm{r}, \bm{r}'$ and costs $\bm{c}, \bm{c}'$.
\begin{align*}
    r_t(\perp) &= r'_t(\perp) = c_t(\perp) = c'_t(\perp) = 0 \\
    r_t(\top) &= \begin{cases}
        1/2 & t \leq \tau \\
        0 & t > \tau
    \end{cases} \\
    r'_t(\top) &= \begin{cases}
        1/2 & t \leq \tau \\
        1 & t > \tau
    \end{cases} \\
    c_t(\top) &= c'_t(\top) = 1
\end{align*}
Note that because the first $\tau$ rounds are identical between the two sequences, the expected spending of the algorithm is the same against both sequences by round $\tau$.
Let $B_\tau$ denote this quantity.
Since the algorithm spends at most $B$ in expectation, we have that the expected reward of the algorithm against the first sequence is exactly $B_\tau / 2$, while its expected reward against the second sequence is at most $B_\tau/2 + (B - B_\tau) = B - B_\tau / 2$.
Meanwhile, the two experts obtain exactly $S_\tau/2$ and $S'_\tau/2$ against the first sequence and exactly $B - S_\tau/2$ and $B - S'_\tau/2$ against the second sequence.
If the algorithm has expected regret at most $R$ on both sequences, then
\[
    \textstyle 2R \geq (S_\tau/2 - B_\tau/2) + ((B - S'_\tau/2) - (B - B_\tau / 2)) = (S_\tau - S'_\tau) / 2 = D/2.
\]
\end{proof}

\begin{theorem}[\cite{braverman2025newbenchmark}]\label{theorem:lower-bound-number-experts-dependence}
Let $B = T/2$.
For all $D \geq 1$ and $F \geq 2$ satisfying $T\geq 2D\log F$, there exists a baseline spending pattern, a set of $F$ experts that are $D$-pacing with respect to this baseline spending pattern, and a sequence of (random) rewards and costs such that the expected regret of any budget-feasible fractional-play algorithm is at least $\Omega(D \sqrt{\log F})$.
This lower bound holds even if the algorithm can observe the rewards and costs of the current round before making its decision.
\end{theorem}

\begin{proof}[Proof Sketch]
We outline the construction and provide a sketch for why it witnesses the desired lower bound.
See \cite{braverman2025newbenchmark} for the full proof.\footnote{We modify their construction slightly by adding padded rounds at the end when $T> 2D\log F$.}

For simplicity, assume $F$ is a power of 2 and $D\in \mathbb{Z}^+$. This is without loss up to constant factors. The construction consists of two non-null actions (so three actions total) and repeats a spend-or-save dilemma in each of $\log F$ windows of length $2D$.
Each window $i \in \{1, \dots, \log F\}$ is split into two halves.
The first action yields reward $a_i \in [0,1]$ in each round in the first half and 0 reward in each round in the second half.
The second action does the opposite: it yields 0 reward in each round in the first half but $b_i \in [0,1]$ reward in each round in the second half.
The first action incurs cost $2B/T=1$ when it yields reward $a_i$ and incurs 0 cost otherwise. The second action incurs cost $2B/T=1$ when it yields reward $b_i$ and incurs 0 cost otherwise. The set of experts is indexed by $\{0,1\}^{\log F}$ where a 0 in coordinate $i$ means that the expert chooses the first action in the first half (and the null action in the second half) and a 1 means that the expert chooses the second action in the second half (and the null action in the first half). If there are any rounds remaining after $\log F$ windows, the reward and cost of each non-null action is 1 and $B/T=1/2$, respectively, and each expert gets reward 1 in each remaining round. This means that any money saved during the $\log F$ windows by the algorithm cannot be spent to catch up to any expert during this time period. Let the baseline spending pattern be the uniform pacing spending pattern, i.e. spending $B/T$ per round.
Then each expert is indeed $D$-pacing against this baseline spending pattern: the cumulative deviation is at most $D$ at any timestep. Let $B_0 = D\log F$ be the cumulative spending of the uniformly pacing spending pattern during the $\log F$ windows.

The rewards $a_1, b_1, \dots, a_{\log F}, b_{\log F}$ are given by an unbiased random walk with step size $\varepsilon = 1 / \sqrt{\log F}$ that starts at $a_1 = 1/2$ and terminates when it hits either 0 or 1.
The expected reward of the best expert is
\[
    \textstyle \opt = D \cdot \EE[\sum_{i=1}^{\log F} \max\{a_i, b_i\}] + (T-2D\log F).
\]
Meanwhile, let $C$ denote the actual cumulative spending of the algorithm within the $\log F$ windows. Then its expected total reward at most 
\begin{align*}
&\quad\ a_1 B_0 + \EE[(C-B_0)_+] + \EE[\min\{T-2D\log F, 2(B-C)\}]\\
&= D \cdot \EE\left [\sum_{i=1}^{\log F} \frac{a_i + b_i}{2}\right] + \EE[(C-B_0)_+] + \EE[\min\{T-2D\log F, 2(B-C)\}]\\
&\leq D \cdot \EE\left[\sum_{i=1}^{\log F} \frac{a_i + b_i}{2}\right]+ (T-2D\log F).
\end{align*}
The key observation behind this fact is that for $\tau\leq 2D\log F$,
\[
    \textstyle Q_\tau = (B_0 - \sum_{t=1}^\tau c_t(x_t)) \cdot (a_{\lceil\tau/2D\rceil} \cdot \1(\text{$\tau$ first half}) + b_{\lceil\tau/2D\rceil} \cdot \1(\text{$\tau$ second half}))) + \sum_{t=1}^\tau r_t(x_t)
\]
remains a martingale, even when $x_t$ is allowed to depend on $r_t$: the future value of money remains equal in expectation to its current value, even if the algorithm knows the current round's reward.

The difference in reward is then
\[
    \textstyle \Omega(D) \cdot \EE[\sum_{i=1}^{\log F} \abs{a_i - b_i}]
\]
Note that with step size $\varepsilon = \Theta(1 / \sqrt{\log F})$, the random walk does not terminate after $2\log F$ steps with constant probability.
It follows that any algorithm's regret is $\Omega(D \sqrt{\log F})$.
\end{proof}

Together with the standard no-regret lower bound, we get that in the OLRC setting, no learner can guarantee regret $o(\max\{D, T^{1/2}\} \cdot \sqrt{\log F})$.

\begin{theorem}[\cite{braverman2025newbenchmark}]
For all $D \geq 1$ and $F \geq 2$ satisfying $T\geq 2D\log F$, there exists a budget $B$, a baseline spending pattern, a set of $F$ experts that are $D$-pacing with respect to this baseline spending pattern, and a sequence of (random) rewards and costs such that the expected regret of any budget-feasible fractional-play learning algorithm is at least $\Omega(\max\{D, T^{1/2}\} \cdot \sqrt{\log F})$.
\end{theorem}

\textbf{AI Disclosure}: 
The authors used OpenAI's ChatGPT extensively in the development of this work. The authors independently came up with an equivalent algorithm for the warm-up example in \Cref{sec:warm-up-example}, and ChatGPT found an equivalent formulation of the algorithm which is generalizable to broader settings. In particular, the main algorithms and analysis in \Cref{sec:core-algorithms} and \Cref{sec:extensions} presented in the paper were generated and written through interactions with ChatGPT. ChatGPT was also used to investigate related literature and improve the writing. The authors subsequently checked and verified all the mathematical claims, algorithms, proofs, and references appearing in the submitted manuscript and take full responsibility for their correctness and originality.

\bibliographystyle{plainnat}
\bibliography{references} 


\appendix

\section{Summary of Guarantees for One Resource}
\label{appendix:summary-of-guarantees}
Suppose there is a single resource, and suppose the parameters 
$\alpha$ and $V$ in \Cref{assumption:soft-conversion} are fixed constants. Let $D$ be the pacing or subpacing width of the experts and assume $D\ge1$, $F\ge2$, and $T\ge1$. Let $B$ be the budget. Write $\Reg_T^f$ for regret against expert $f$ and
$\Reg_T:=\max_f\Reg_T^f$.  The notation $\widetilde O_T$ only suppresses polylogarithmic factors in $T$, so all dependence on $F$ remains explicit.  Under
these simplifications, the main guarantees are:
\begin{itemize}
    \item \textbf{Fractional play, $D$-pacing/subpacing experts.} \Cref{algorithm:core-ORA} and \Cref{algorithm:core-OLRC} (and optionally with Extension 1 in \Cref{subsec:subpacing-experts}) satisfy fractional budget feasibility, and simultaneously for every expert
    $f$,
    \[
      \Reg_T^f=O\!\left(D\sqrt{\log F}\right)
      \qquad\text{with pre-decision feedback},
    \]
    and
    \[
      \Reg_T^f=O\!\left((\sqrt T+D)\sqrt{\log F}\right)
      \qquad\text{without pre-decision feedback}.
    \]
    Note that the same rates hold for $D$-pacing and $D$-subpacing experts, by \Cref{thm:raising}.

    \item \textbf{Fractional play, subpacing learning algorithm.}
   \Cref{algorithm:core-ORA} and \Cref{algorithm:core-OLRC} with Extension 2 in \Cref{subsec:extension-subpacing-algo} satisfy fractional budget feasibility and
    achieve
    \[
      \Reg_T^f= O\!\left(D\sqrt{\log F}\right)
      \qquad\text{with pre-decision feedback},
    \]
    and
    \[
      \Reg_T^f= O\!\left((\sqrt T+D)\sqrt{\log F}\right)
      \qquad\text{without pre-decision feedback}.
    \]
    Their fractional spending path can be made $W$-subpacing, with
    \[
      W=\widetilde O_T\!\left(D\sqrt{\log F}\right)
      \qquad\text{with pre-decision feedback},
    \]
    and 
    \[
      W=\widetilde O_T\!\left(
      D+\frac{D^2\sqrt{\log F}}{\sqrt T+D}
      \right)
      \qquad\text{without pre-decision feedback}
    \]

    \item \textbf{Sampled play under full-information feedback.}
    \Cref{algorithm:core-ORA} and \Cref{algorithm:core-OLRC} with Extension 3 in \Cref{sec:sampled-feasibility} satisfy the budget constraint with sampled actions, and  with probability at least
    $1-\delta$, it adds at most
    \[
      O\!\left(
      \sqrt{\bigl(1+\min\{B,T\}\bigr)\log\frac1\delta}
      +\log\frac1\delta
      \right)
    \]
    to the fractional regret bounds.  If the fractional path is
    $W$-subpacing, the realized sampled path has the same $W$ plus a
    deviation of this order.


\end{itemize}

\end{document}